\documentclass[journal]{IEEEtran}
\newcommand{\model}{InductWave}
\newcommand{\msgpass}{WAVBFNet}
\usepackage{xurl}
\usepackage{amsmath}
\usepackage{cite}
\usepackage{multirow}
\usepackage{amssymb}
\usepackage{graphicx}
\usepackage{adjustbox}
\usepackage{booktabs}
\usepackage{amsthm}
\newtheorem{theorem}{Theorem}
\usepackage{todonotes}
\usepackage{array}

\begin{document}

\title{\model: Inductive Multi-Hop Logical Query Answering on Knowledge Graphs}

\author{Mayank Kharbanda, Michael Cochez, Rajiv Ratn Shah, Raghava Mutharaju
\thanks{M. Kharbanda (mayankk@iiitd.ac.in) is with IIIT-Delhi, India, and guest at Vrije Universiteit, Amsterdam, The Netherlands. M. Cochez (michael.cochez@abo.fi) is with Ellis Institute Finland and Åbo Akademi University, Finland, prior with Vrije Universiteit, Amsterdam, The Netherlands. R. Shah (rajivratn@iiitd.ac.in) is with IIIT-Delhi, India and R. Mutharaju (raghava@iitpkd.ac.in) is with IIT Palakkad, Kerala, India.}}
\markboth{}
{Mayank Kharbanda, Michael Cochez, Rajiv Ratn Shah, Raghava Mutharaju \MakeLowercase{\textit{et al.}}: \model: Inductive Multi-Hop Logical Query Answering on Knowledge Graphs}
\maketitle

\begin{abstract}
Logical Multi-Hop Query Answering over Knowledge Graphs (KGs) can be formulated as querying, with an implicit completeness assumption. Current works mainly focus on Existential First Order Logic (EFO) queries. These EFO queries contain conjunction($\wedge$), disjunction ($\vee$), and negation ($\neg$) operators. 
Most existing works employ \emph{transductive} reasoning, meaning they are not capable of reasoning over entities unseen during training. In the real world, there is a resource scarcity, and we cannot train a model with all the nodes of a large KG. Hence, we propose \model{}, a wavelet-based \emph{inductive} embedding method for logical query answering on large KGs. Here, the training graph consists of fewer nodes than the test graph. Our model performs on par with the baseline models while having half the number of message-passing layers. It outperforms all of them in most cases, with $75\%$ of the layers. These fewer resource requirements enable us to evaluate \model{} on massive graphs, such as Wiki-KG. We test our model using extensive experiments across varying train-test graph proportions of the FB15k-(237) dataset, comparing it with the state-of-the-art models. The code and datasets for the model are available at \url{https://github.com/kracr/inductwave/}.
\end{abstract}


%
\IEEEpeerreviewmaketitle

\section{Introduction}
\IEEEPARstart{A}{} Knowledge Graph (KG)~\cite{hogan2021knowledge} is a directed graph used to represent facts. It is a set of triples in the form of source, relation, and object. The KGs are used to extract non-trivial information from data by leveraging structural and logical features. These graphs encompass diverse domains, including healthcare, finance, e-commerce, and search. Tasks such as recommendation systems, link prediction, and knowledge retrieval are performed on KGs to extract novel information~\cite{ji2021survey}. 

Multi-hop logical query answering over KGs involves answering First Order Logic (FOL) queries. It includes traversing more than one hop from a starting node in the KG. Current works mainly focus on Existential First Order Logic (EFO) queries consisting of conjunction ($\wedge$), disjunction ($\vee$), and negation ($\neg$) operators. For query answering, there are two primary ways to train a model. First is the \emph{transductive} method, in which the nodes and relations in the training and test graphs are identical; only the number of triples (edges) in the test graph increases. The other is the \emph{inductive} method, in which the model is trained on a subset of nodes and/or relations and can handle new nodes and/or relations at test time.

\textbf{Current State-Of-The-Art (SOTA).}
Traditional query-answering languages, such as SPARQL, become inadequate while processing queries over incomplete or noisy data. To address this, neural logical query answering methods have been introduced. These models embed queries and the KG in a latent space and predict answers against noise and missing links.

There has been significant progress in recent years in neural methods for multi-hop logical query answering. At the same time, most of these methods are \emph{transductive}. These models require training across all parts of the KG and often fail when encountering new nodes/relations at inference time. As knowledge and datasets expand in the real world, there is a need to process queries on massive KGs containing millions of nodes. However, due to resource constraints, it is often not feasible to train with all the nodes of these enormous KGs.

One way to address this issue is to train a model on a subgraph with fewer nodes and then extrapolate it to process queries on the larger graph. \emph{Inductive} methods, such as GNN-QE~\cite{zhu2022neural} and NodePiece-QE, have been proposed to incorporate this idea, and they have outperformed \emph{transductive} models~\cite{galkin2022inductive,galkin2021nodepiece}.

In an \emph{inductive} setting, GNN-QE generally performs better than NodePieceQE for small or medium-sized graphs. However, GNN-QE employs a memory-intensive link-prediction method, NBF-Net~\cite{zhu2021neural}. NBF-Net makes training on large graphs challenging. In contrast, NodePiece-QE does not face such memory constraints, enabling it to handle larger graphs efficiently.

We introduce \model{}, a graph wavelet-based method for logical query answering. It leverages the structural information of nodes to strengthen link prediction. As a result, the model achieves capabilities similar to GNN-QE with fewer message-passing layers. Ultimately, this allows us to use a message-passing method to query large graphs with millions of nodes (Wiki-KG).

\textbf{Contribution I: A novel method, \msgpass{}, for link prediction.}
We propose a novel message passing algorithm, \msgpass{}. The method combines Graph Wavelet embeddings~\cite{donnat2018learning} with the Neural Bellman-Ford Network (NBF-Net)~\cite{zhu2021neural} for link prediction. The former provides the structural context of a node to the latter message-passing method. \msgpass{} is used for the relation projection operation in the query answering process.

\begin{figure*}[!t]
    \centering
    \includegraphics[width=\textwidth]{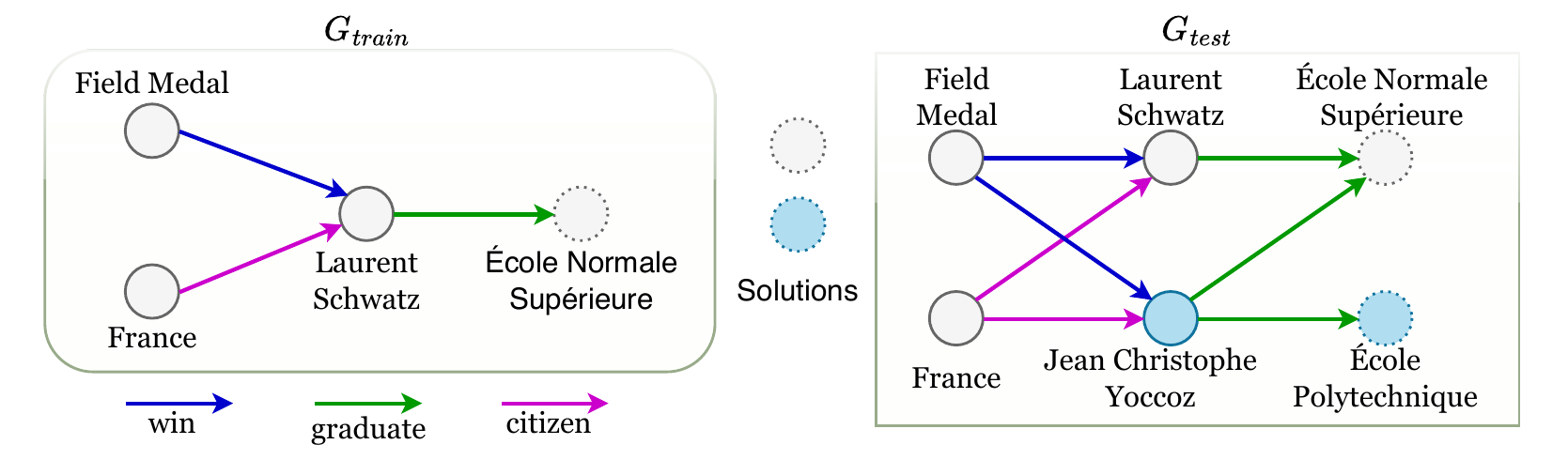}
    \caption{Toy example for \emph{inductive} query answering where the test graph ($G_{test}$) contains additional nodes (represented by blue color) than the train graph ($G_{train}$). Dotted circles represent the solution for the query, \texttt{Q1} - \emph{Name the university from which a French Field medalist graduated}, with FOL as $Q=v.\exists u: win(Field Medal, u) \wedge citizen(France,u) \wedge graduate(u,v)$}
    \label{fig:induct_example}
\end{figure*}
\textbf{Contribution II: Efficient execution of the message passing.}
We extend the GE-SpMM~\cite{huang2020ge} method to make it compatible with graph wavelet embedding. This enables efficient computation of \msgpass{} on GPU hardware. It reduces the memory complexity of \msgpass{} from $O(2b|\mathcal{E}|d)$ to $O(b|\mathcal{V}|d+|\mathcal{E}|d)$. Here, $b$ is the batch size, $|\mathcal{E}|$ and $|\mathcal{V}|$ are the triple and node counts, and $d$ is the embedding dimension.

\textbf{Contribution III: Extensive Evaluation.}
We evaluate our model on a diverse proportion of training and inference nodes set from the FB15k-(237)~\cite{toutanova2015observed} dataset. We also test our model on the Wiki-KG~\cite{hu2020open} dataset, which contains millions of nodes. The results are supported by an ablation study of \model{}, along with its space and runtime analysis.

\section{Related Work}
\textbf{Multi-Hop Query Answering.}
In multi-hop reasoning, one approach is to traverse a path in a KG to perform link prediction~\cite{xiong2017deeppath,lin2023multi,chen2019embedding,guo2018knowledge,guo2016jointly,wang2020entity}. These techniques improve the prediction for rare or complex relations. Another application of multi-hop reasoning is in answering complex logical queries. This involves processing FOL operators to obtain the answers. Our work aligns with this latter approach.

Graph Query Embedding (GQE)~\cite{hamilton2018embedding} and Query2Box~\cite{ren2020query2box} were the first to introduce methods to resolve queries containing disjunction ($\vee$) and conjunction ($\wedge$) operators. Using the beta distributions for embedding, BetaE~\cite{ren2020beta} incorporated the negation operator ($\neg$) in the queries. There are geometric embedding methods, such as ConE~\cite{zhang2021cone} and Query2Geom~\cite{sardina2022analysis}, which embed queries as geometric shapes in the latent space. FuzzyQE~\cite{chen2022fuzzy} provides a fuzzy set of answers between the query traversal. GNN-QE~\cite{zhu2022neural} uses NBF-Net~\cite{zhu2021neural} for the relation projection and fuzzy operators for other FOL operations. CQD~\cite{arakelyan2020complex} employs a greedy approach of beam search for FOL operations, and ComplEx~\cite{trouillon2016complex} for relational projection. RConE~\cite{kharbanda2025rcone} and STARQE~\cite{alivanistos2021query} handle the query answering in multi-modal and hyper-relational graphs, respectively. The detailed studies about the KG reasoning and logical query answering are in~\cite{liang2022survey} and~\cite{ren2023neural}, respectively.

{\setlength{\parindent}{0pt}\textbf{Inductive Logical Query Answering.}
Most logical query-answering methods discussed so far are \emph{transductive}. These require training queries for each entity in the KG. To generalize the training process, \emph{inductive} methods have been proposed.}

GNN-QE~\cite{zhu2022neural}, initially proposed for \emph{transductive} reasoning, can also be used for \emph{inductive} query answering~\cite{galkin2022inductive}. Since, for each relation projection, the node and relation embeddings are initialized based on the query. The model is trained on a sub-graph of the KG. NodePiece-QE~\cite{galkin2022inductive,galkin2021nodepiece} is another inductive method. It represents each node by its incoming (outgoing) relations. It captures high-level structural information through the node’s distance from a few predetermined anchor nodes. The model utilizes CQD~\cite{arakelyan2020complex} for FOL operations. 

ULTRA~\cite{galkin2023towards,galkin2024zero} learns generalized embeddings across multiple KGs. It trains on the relational structures of a few KGs and tests on an entirely new KG by comparing these structures. The model effectively caters to a new set of relations at test time.

{\setlength{\parindent}{0pt}\textbf{Wavelets in Graphs.}
GraphWave~\cite{donnat2018learning} generates diffusion wavelets on an undirected graph for node embeddings. It captures the structural information of a node’s neighborhood via heat-kernel-based information flow. GWNN~\cite{xu2019graph} proposes a Graph Wavelet Neural Network for node classification on an undirected graph. The model learns a diagonal filter to tune information from a neighbor’s wavelet transformation. While previous models focused on low-pass filters, ASWT~\cite{liu2024aswt} uses both band-pass and low-pass filters for graph wavelets. It combines GCN and Graph wavelets for node classification.}

\model{} is an \emph{inductive}, logical query-answering method based on message passing, similar to GNN-QE. However, it requires less memory, making it suitable for large graphs. The method uses graph wavelets in its framework. It trains on a sample of the KG and evaluates on the entire KG; thus, we use GNN-QE and NodePiece-QE as baselines. We exclude ULTRA from our comparison, because the method has fundamentally different goal. It generalizes query answering across multiple KGs, by training on a set of KGs. 

\section{Preliminaries}\label{sec:prelimianries}
\textbf{Inductive Reasoning on KG.} A Knowledge Graph $G(\mathcal{V},\mathcal{E},\mathcal{R})$ is a directed graph with a set of nodes $\mathcal{V}$, a set of relations $\mathcal{R}$, and a triple set $\mathcal{E}$.
\begin{equation}
    \mathcal{E} = \{(e_s, r, e_o) | e_s, e_o \in \mathcal{V}, r \in \mathcal{R}\}
\end{equation}
Given $|\mathcal{V}|=N$ and $|\mathcal{R}|=M$, we say the query answering is \emph{inductive} if the training KG ($G_{train}$) contains $\alpha_1$ nodes and $\beta_1$ relations, such that $|\alpha_1| < N$ and/or $|\beta_1| < M$, i.e., the training nodes/relations are a subset of $G$. The test graph ($G_{test}$) contains $\alpha_2$ nodes and $\beta_2$ relations, such that $\alpha_2 \setminus \alpha_1 \neq \phi$ and/or $\beta_2 \setminus \beta_1 \neq \phi$. For our work, we train on all relations ($|\beta_1|=M,\;\; |\alpha_1|<N$).

Figure \ref{fig:induct_example} provides a toy example. Given a query, \texttt{Q1} - \emph{Name the university from which a French Field medalist graduated.} In the training graph ($G_{train}$), we get \emph{École Normale Supérieure} as the answer. While in the test graph ($G_{test}$), we have two new nodes: \emph{Jean Christophe Yoccoz} and \emph{École Polytechnique}. Hence, we obtain an additional answer for the same query as \emph{École Polytechnique}.

{\setlength{\parindent}{0pt}\textbf{First Order Logic.}
First Order Logic (FOL) query $Q$ utilizes conjunction ($\wedge$), disjunction ($\vee$), existential quantification ($\exists$), and negation ($\neg$) as its logical operators. We exclude universal quantification ($\forall$) because it is rarely employed in real-world KGs, as noted in~\cite{ren2020beta}.}  

The query $Q$ has a distinguished target variable $V_?$ (answer to the query), a finite set of existentially quantified bound variables $V_1, \ldots, V_k$, and a constant entity set $\mathcal{V}_a\subseteq \mathcal{V}$. The queries are structured in Disjunctive Normal Form (DNF) to accommodate the union operator at the conclusion, thereby ensuring scalability for more complex queries.
\begin{equation}
    Q[V_?]=V_?.\exists V_1, \dots, V_k : c_1 \vee c_2 \vee ... \vee c_n
\end{equation}
Each clause $c_i$ is a conjunction of one or more relational binary functions (positive or negative) $r(x,y)$, $r \in \mathcal{R}$, where $x$ and $y$ can be a constant or a variable. For query \texttt{Q1}, the FOL would be $Q=v.\exists u: win(Field Medal, u) \wedge citizen(France,u) \wedge graduate(u,v)$.

{\setlength{\parindent}{0pt}\textbf{Fuzzy Set.} 
A Fuzzy set is a relaxed conventional set, defined as $A = \{(x,\mu)| x \in U\}$. $\mu \in [0,1]$ is the membership function of element $x$ for the set $A$. $U$ is the universal set. Applying FOL operators on a fuzzy set of nodes relaxes the membership of intermediate answers of a sub-query and enhances interpretability~\cite{zhu2022neural}.}

\section{\model{}}
This section describes the three components of \model{} that handle FOL queries, such as \texttt{Q1} in Figure \ref{fig:induct_example}. The first component, \textit{Relation Projection}, processes different relations specified in the query. For example, in \texttt{Q1}, it executes the relations \emph{win}, \emph{graduate}, and \emph{citizen}. The second component, \textit{Fuzzy Operations}, provides details about how \model{} handles other FOL operators, including conjunction ($\wedge$), disjunction ($\vee$), and negation ($\neg$). The third component, \textit{Training Method}, consists of additional information about the model, including traversal dropout and the loss function. All FOL operators are executed on a fuzzy set of entities to improve interpretability inbetween query execution~\cite{zhu2022neural}.


\subsection{Relation Projection}\label{sec:relation_projection}
Given a fuzzy set of head entities $\mathbf{f}_{he}$ and a query relation $q$, the relation projection operation provides a fuzzy set of all the tail entities $\mathbf{f}_{ta}$ that can be reached from $\mathbf{f}_{he}$ using $q$ (Equation \ref{eqn:relproj_gen}). In \texttt{Q1}, for example, for the relation \emph{win}, $\mathbf{f}_{he}$ is a one-hot vector with the probability of Field-Medal being $1$. The relation projection $\mathcal{P}_q (\mathbf{f}_{he})$ for \emph{win} will produce $\mathbf{f}_{ta}$, where the probability of all the individuals who \emph{win} the Fields Medal (in the graph) approaching $1$.
\begin{equation}\label{eqn:relproj_gen}
    \mathcal{P}_q (\mathbf{f}_{he}) : [0,1]^{\mathcal{V}} \mapsto [0,1]^{\mathcal{V}}
\end{equation}
We present \msgpass{} for computing relation projections. The approach consists of two components: graph wavelets and NBF-Net. First, we introduce the \textit{Knowledge Graph Laplacian}, which serves as a foundation for constructing graph wavelets. In the second part, we use the KG Laplacian to develop \textit{Graph Wavelet Embedding}. Finally, in the \textit{\msgpass{}} section, we combine wavelet embeddings with NBF-Net to create a relation projection function.

\subsubsection{Knowledge Graph Laplacian}
To generate the graph wavelet embeddings for \msgpass{}, we first need to calculate the KG Laplacian. A graph Laplacian measures the information flow through edges over time. The Laplacian is defined for a simple undirected graph, as its adjacency matrix is symmetric.

Directed graphs are asymmetric and, therefore, do not have a standard Laplacian. To address this, MagNet~\cite{zhang2021magnet} introduced a magnet Laplacian for directed graphs. It represents the asymmetric adjacency matrix as a hermitian matrix, that preserves the directional information of each edge. Similar to the standard Laplacian, the Laplacian derived from this hermition matrix is positive-semidefinite. Hence, the magnet Laplacian has an orthonormal basis of eigenvectors associated with non-negative eigenvalues.

In a KG, each edge has a direction and a specific relation associated with it. The magnet Laplacian~\cite{zhang2021magnet} is already defined for directed graphs. We extend the method and define a KG Laplacian that accounts for both direction and relation.

Let $\mathbf{A}^r$ be an adjacency matrix containing all the edges for relation $r$ ($r \in\mathcal{R}$) of the KG, $G_{X}$ ($G_{X}=G_{train}/G_{valid}/G_{test}$ dataset). Let $\mathbf{A}_s^r$ be the symmetric matrix generated from $\mathbf{A}^r$, with a degree matrix $\mathbf{D}^r_s(u,u)$, defined as 
\begin{align}
    \mathbf{A}^r_s(u,v) &= \frac{1}{2} (\mathbf{A}^r(u, v) + \mathbf{A}^r(v, u))\\
    \mathbf{D}^r_s(u,u) &= \sum_{v \in \mathcal{V}} \mathbf{A}_s^r(u, v)
\end{align}
$\mathbf{A}^r_s(u,v)>0$ indicates that a relation $r$ exists between node $u$ and $v$. The direction of this edge is captured by
\begin{align}
    \mathbf{\Theta}_{r}^{(g)} (u, v) &= 2\pi g (\mathbf{A}^r(u, v) - \mathbf{A}^r(v, u)) \label{eqn:thetaf}
\end{align}
where $g$ is a hyper-parameter with values in the range $[0, 0.25]$. If there is a direct edge $\mathbf{A}^r(u, v)\in \mathcal{E}$, but not the reverse edge $\mathbf{A}^r(v, u) \notin \mathcal{E}$, then $\mathbf{\Theta}_{r}^{(g)}(u,v)$ will be positive while $\mathbf{\Theta}_{r}^{(g)}(v,u)$ will be negative. Conversely, if the relational edge is bi-directional or does not exist, then $\mathbf{\Theta}_{r}^{(g)} (u, v) =0$. We define the original relational adjacency matrix $\mathbf{A}^r$, using a hermitian matrix as
\begin{equation}
    \mathbf{H}^{r(g)} = \mathbf{A}_s^r \odot \exp{(i_r\mathbf{\Theta}_r^{(g)})}    
\end{equation}
here, $i_r$ is the imaginary dimension for relation $r$, and $\odot$ is element-wise multiplication. $\mathbf{A}_s^r$ captures edge existence, while $\exp{(i_r\mathbf{\Theta}_r^{(g)})}$ accounts for the direction. Inspired by~\cite{zhang2021magnet}, we propose the unnormalized Laplacian for relation $r$ as 
\begin{align}\label{eqn:un_laplacian}
    \mathbf{L}_{un}^{r(g)} = \mathbf{D}_s^r - \mathbf{H}^{r(g)} = \mathbf{D}_s^r - \mathbf{A}_s^r \odot \exp{(i_r\mathbf{\Theta}_r^{(g)})}
\end{align}
We define KG Laplacian as a set of these relational Laplacians, $\mathbf{L}_{un}^{(g)}=\{\mathbf{L}_{un}^{r(g)}|r\in \mathcal{R}\}$.
The proof of $\mathbf{L}_{un}^{r(g)}$ being positive semi-definite can be found in the Supplementary Material. 

In Equation \ref{eqn:un_laplacian}, the $r^{th}$ Laplacian contains information about that specific relation $r$. However, there is no information exchange between two distinct relational Laplacians $\mathbf{L}_{un}^{r(g)}$, $\mathbf{L}_{un}^{s(g)}$ ($r,s \in \mathcal{R}, r \ne s$). This may lead to loss of inter-relational context within a KG. To address this information loss, we normalize the relational Laplacian using the degree of a node in the entire KG, $G_X$, rather than just based on relation $r$. The normalized magnet Laplacian for relation $r$ is defined as
\begin{align}
    \mathbf{L}_{n}^{r(g)} &= \mathbf{D}_z^{-1/2}\mathbf{L}_{un}^{r(g)}\mathbf{D}_z^{-1/2} \nonumber\\ 
    &= \mathbf{D}_z^{-1/2}(\mathbf{D}^r_s - \mathbf{A}^r_s \odot exp(i_r\mathbf{\Theta}_r^{(g)}))\mathbf{D}_z^{-1/2}\label{eqn:n_laplacian}
\end{align}
where $\mathbf{D}_z$ is the sum of all degree matrices $\mathbf{D}^r_s$, $\forall r\in \mathcal{R}$, defined as
\begin{equation}\label{eqn:dz}
    \mathbf{D}_z(u,u) = \sum_{r\in\mathcal{R}}\sum_{v \in \mathcal{V}} \mathbf{A}^r_s(u, v)
\end{equation}
We use this normalized KG Laplacian (Equation \ref{eqn:n_laplacian}) to generate graph wavelet embedding, which is subsequently used in \msgpass{} for relation projection.

\subsubsection{Graph Wavelet Embedding}
In this section, we generate graph wavelet embeddings for \msgpass{} using the KG Laplacian (Equation \ref{eqn:n_laplacian}). GraphWave~\cite{donnat2018learning} proposed a diffusion wavelet embedding method for extracting structural information around a node in an undirected graph. We adapt this concept for a directed graph with $\mathcal{R}$ relations, i.e., a KG.

The KG Laplacian (Equation \ref{eqn:n_laplacian}) is different from the conventional undirected Laplacian, since it involves complex numbers rather than just the real numbers. The spectral decomposition of the normalized KG Laplacian for a relation $r$ is expressed as $\mathbf{L}_{n}^r = \mathbf{U}^r \Lambda^r \mathbf{U}^{r\dagger}$ (omitted $g$ from $\mathbf{L}_{n}^{r(g)}$ for simplicity). Here, $\mathbf{U}^r$ consists of eigenvectors ($\mathbf{U}^{r\dagger}$ is the conjugate transpose of $\mathbf{U}^r$), and $\mathbf{\Lambda}^r = Diag\{\lambda_1,...\lambda_n\}$ is the diagonal matrix containing ordered eigenvalues.

Let $\mathfrak{g}_s$ be a low-pass heat filter, $\mathfrak{g}_s(\lambda) = e^{-\lambda s}$. Here, the scaling factor s is used to adjust the extent to which a distant neighbor of a node affects its embedding. The graph spectral wavelet for a relation $r$ is defined as
\begin{align}\label{eqn:graph_wavelet}
     \mathbf{\Psi}^r &= \mathbf{U}^r \mathfrak{g}_s (\mathbf{\Lambda}^r)\mathbf{U}^{r\dagger}
\end{align}
The low-pass filter ($\mathfrak{g}_s$) incorporates only the lower eigenvalue of a Laplacian, favoring smoothness in the signal. This smoothness leads to neighboring nodes having similar embeddings (graph homophily). However, computing $\mathbf{\Psi}^r$ (Equation \ref{eqn:graph_wavelet}) is an expensive operation. Hence, to simplify this process, we use the Chebyshev polynomial approximation~\cite{shuman2011chebyshev}.

The graph wavelet embedding is obtained by sampling from the characteristic function of $\mathbf{\Psi}^r$. This embedding is used to measure the expected amount of information a node $u$ receives from all other nodes, in relation $r$. The characteristic function for node $u$ is defined as
\begin{align}
   \phi^r_u (t_j,t_k) &= E[e^{i(t_j\mathbf{\Psi}^r_{u}(re)+t_k\mathbf{\Psi}^r_{u}(im))}]\nonumber\\
   &=\frac{1}{\mathcal{|V|}}\sum_{v\in \mathcal{V}}e^{i(t_j\mathbf{\Psi}^r_{vu}(re)+t_k\mathbf{\Psi}^r_{vu}(im))} 
\end{align}
where $\mathbf{\Psi}^r_{u}(re)$ and $\mathbf{\Psi}^r_{u}(im)$ represent the real and imaginary parts of $\mathbf{\Psi}^r_{u}$. $t_j, t_k \in \mathbb{R}$ are the uniform random samples taken from the joint distribution. The graph wavelet embeddings with $d/2$ samples is defined as
\begin{equation}\label{eqn:wave_emb}
    \mathbf{\chi}^r_u=[Re(\phi^r_u (t_j, t_k)), Im(\phi^r_u (t_j,t_k))]_{(t_{1j}, t_{1k}),\ldots,(t_{(d/2)j}, t_{(d/2)k})}
\end{equation}
The resulting graph wavelet embeddings, $\mathbf{\chi}^r_u$, will be $d$-dimensional, with $d/2$ real and $d/2$ imaginary components. We use this wavelet embedding (Equation \ref{eqn:wave_emb}) for the \msgpass{} in the following section.

\subsubsection{\msgpass{}}
The relation projection method, \msgpass{}, is constructed using two components: the graph wavelet embeddings and the Neural Bellman Ford Network (NBF-Net)~\cite{zhu2021neural}.

GNN-QE~\cite{zhu2022neural} uses NBF-Net for the relational projection. NBF-Net initializes the embeddings for the source node and the query relation. It then performs message passing to obtain the embeddings of all the nodes. These embeddings are subsequently used to compute the probability of each node being the tail (for the given source node and relation) in the link prediction task. GNN-QE modifies NBF-Net to operate with a fuzzy set of source nodes instead of a single source node for the projection operation.

Our model, \msgpass{}, integrates graph wavelet embeddings (Equation \ref{eqn:wave_emb}) with NBF-Net, rather than relying solely on the latter. This addition of graph wavelet embedding provides extra structural information. We will see in the results section that combining the two helps in early convergence for our model, reducing the number of message-passing layers required. 

In WAVBFNet, we begin by initializing the source nodes' embeddings (Equation \ref{eqn:mp_init}).
\begin{equation}\label{eqn:mp_init}
    \mathbf{h}^{(0)}_v \leftarrow x_v \mathbf{q}
\end{equation}
here, $x_v$ is the probability of node $v$ in $\mathbf{f}_{he}$ (fuzzy set of head nodes). Since we are working with complex wavelet embedding, we split the vector (Equation \ref{eqn:mp_init}) into real and imaginary parts and handle them separately. $\mathbf{h}^{(0)}_{v_{re}} = \mathbf{h}^{(0)}_{v_{[:i/2]}}$ being the real vector and $\mathbf{h}^{(0)}_{v_{im}} = \mathbf{h}^{(0)}_{v_{[(i/2)+1:]}}$ the imaginary part.

In the message passing method, we define the $MESSAGE$ for a node $v$ as a function of its neighbor's embedding ($\mathbf{h}^{(t-1)}_{x}$), the relation embedding ($w_{q}(x, r, v)\; | \; (x,r,v) \in \mathcal{E}$), and the wavelet embeddings of the neighbor ($\mathbf{\chi}^r_{x}$ from Equation \ref{eqn:wave_emb}). The function for the real part is defined as
\begin{align}
    \label{eqn:mp_message}
    MES&SAGE (\mathbf{h}^{(t-1)}_{x_{re}}, w_{q_{re}}(x, r, v),\mathbf{\chi}^r_{x_{re}}) \nonumber\\
    &= \mathbf{h}^{ (t-1)}_{x_{re}} \odot (\mathbf{W}_r\mathbf{q} + \mathbf{b}_r)_{re} \odot (\mathbf{w_1}+\mathbf{w_2}\mathbf{\chi}^r_{x_{re}})
\end{align}
The term $(\mathbf{W}_r\mathbf{q} + \mathbf{b}_r)_{re}$ indicates that the relation embedding is dependent on the query relation $q$. The product of the first two terms in Equation \ref{eqn:mp_message} is analogous to the link prediction method DistMult~\cite{yang2014embedding}. The expression $(\mathbf{w}_1+\mathbf{w}_2\mathbf{\chi}^r_{x_{re}})$ represents a linear combination of the former two terms, with and without the wavelet embedding ($\mathbf{\chi}^r_{x_{re}}$). A corresponding function can be deduced for the imaginary part as well. Note that we tested other link prediction methods, including ComplEx~\cite{trouillon2016complex} and TransE~\cite{bordes2013translating}, but DistMult performed the best among the three. The message passing for the \msgpass{} (Figure \ref{fig:message_passing}) is described as
\begin{align}
    \mathbf{h}^{(t)}_{v_{re}} \leftarrow &AGG (\{ MESSAGE (\mathbf{h}^{(t-1)}_{x_{re}}, w_{q_{re}}(x, r, v), \mathbf{\chi}^r_{x_{re}}) \mid \nonumber\label{eqn:msgp_re}\\
    &(x, r, v) \in \mathcal{E}(v)\} \cup \{\mathbf{h}^{(0)}_{v_{re}}\})\\
    \mathbf{h}^{(t)}_{v_{im}} \leftarrow &AGG (\{ MESSAGE (\mathbf{h}^{(t-1)}_{x_{im}}, w_{q_{im}}(x, r, v), \mathbf{\chi}^r_{x_{im}}) \mid \nonumber\label{eqn:msgp_im}\\
    &(x, r, v) \in \mathcal{E}(v)\} \cup \{\mathbf{h}^{(0)}_{v_{im}}\})\\
    \mathbf{h}^{(t)}_{v} \leftarrow &CONCAT(\mathbf{h}^{(t)}_{v_{re}}, \mathbf{h}^{(t)}_{v_{im}})
\end{align}
Note, we have concatenated the real and imaginary embeddings ($\mathbf{h}^{(t)}_{v_{re}}$ and $\mathbf{h}^{(t)}_{v_{im}}$) at the end. We use PNA (Principal Neighborhood Aggregator)~\cite{corso2020principal} as the aggregator function ($AGG$), unless specified otherwise.
\begin{figure}
    \centering
    \rotatebox{90}{\includegraphics[width=\columnwidth]{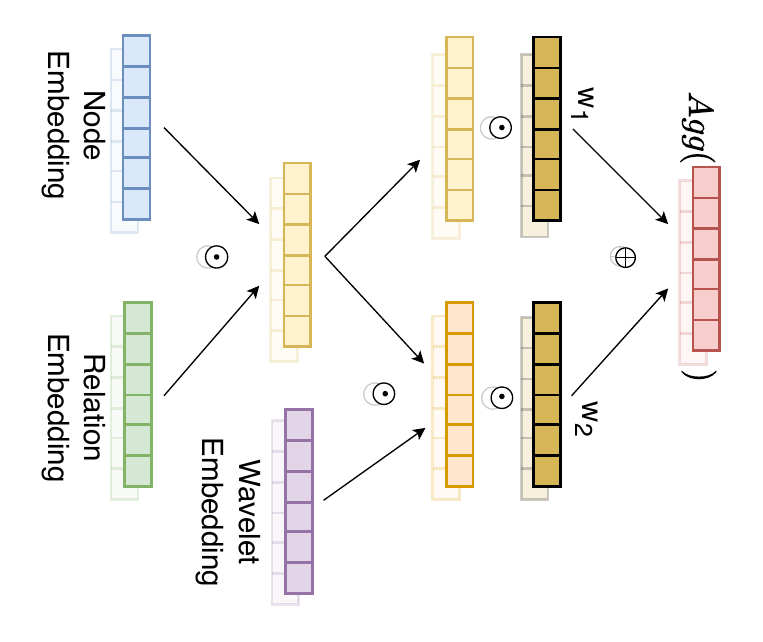}}
    \caption{\msgpass{} message passing. Here $\odot$ is element-wise multiplication, $\oplus$ is element-wise addition, $w_1$ and $w_2$ are parameters. The shadow under each vector and operator represents the real and imaginary parts working in parallel.}
    \label{fig:message_passing}
\end{figure}
After the message passing phase, the embedding is fed into a two-layer feed-forward network to obtain the fuzzy set of tail nodes ($\mathbf{f}_{ta}$). Thus, the relation projection described in the Equation \ref{eqn:relproj_gen} can be expressed as the following function
\begin{equation}
    \mathcal{P}_q(\mathbf{f}_{he}) = \sigma (FF(\mathbf{h}^{(T)})
\end{equation}
where $\sigma$ is the sigmoid activation function that maps all values to the range $[0, 1]$. 

\subsection{Fuzzy Operations}\label{sec:fol_operators}
We employ fuzzy set operations to implement the other FOL operators, following a similar approach to~\cite{zhu2022neural}. Since, these operators should ideally adhere to the laws of logic, such as commutativity, associativity, and closure.

Given the fuzzy sets of nodes from the relation projection as $\mathbf{y_1},\mathbf{y_2} \in [0,1]^{\mathcal{V}}$. The conjunction ($\wedge$), disjunction ($\vee$), and negation ($\neg$) operations are represented by the following formulations, respectively.
\begin{align}
    \mathcal{C}(\mathbf{y_1}, \mathbf{y_2}) &= \mathbf{y_1} \odot \mathbf{y_2}\\
    \mathcal{D}(\mathbf{y_1}, \mathbf{y_2}) &= \mathbf{y_1} + \mathbf{y_2} - \mathbf{y_1} \odot \mathbf{y_2}\\
    \mathcal{N} (\mathbf{y_1}) &= \mathbf{1} - \mathbf{y_1}
\end{align}
where $\odot$ is the element-wise multiplication and $\mathbf{1}$ is a vector of all ones.
\subsection{Training Method}
\textbf{Traversal Dropout.}
Direct query-relation edges originating from the source nodes are randomly removed from the training graph with a probability of $p$~\cite{zhu2022neural}. This prevents the model from overfitting by limiting direct access to the trivial tail entities. Randomly dropping edges from the KG makes the model resilient to the incomplete nature of real-world graphs.

{\setlength{\parindent}{0pt}\textbf{Loss Function.}
The model is trained using the binary cross-entropy loss. Since \model{} generates probability estimates for each node in the graph, negative sampling is not required. Instead, the probabilities for both positive and negative answers are directly included in the loss calculation. The loss is given as}
\begin{align}\label{eqn:loss1}
    L=& -\frac{1}{|\mathcal{A}_Q|}\sum_{u\in \mathcal{A}_Q} log\;p(u|Q)\nonumber\\
    &-\frac{1}{|\mathcal{V}\setminus \mathcal{A}_Q|} \sum_{u'\in \mathcal{V}\setminus \mathcal{A}_Q} log\;(1 - p(u'|Q))
\end{align}
here, $\mathcal{A}_Q$ is the set of all the answers to the complex query $Q$, and $log\;p(u|Q)$ denotes the probability of the answer $u$, if query $Q$ is passed through \model{}.
\section{Experiments}
\subsection{Experimental Setup}
\textbf{Datasets and Evaluation Metric.}
We use FB15k-(237)~\cite{toutanova2015observed} and Wiki-KG~\cite{hu2020open} datasets in our study. The FB15k-(237) dataset is derived from Freebase~\cite{bollacker2008freebase}. It contains 237 relations and improves on the previous version by removing easy inverse relation leakage~\cite{bordes2013translating}. The dataset is used to assess the models with varying train-test graph proportions. The Wiki-KG dataset is part of the Open Graph Benchmark suite, which consists of realistic, large-scale datasets. The dataset features over $2.5M$ nodes in its complete graph. This makes it an ideal choice for testing the models on a large graph. We evaluate all the models using the HITS scores.
\begin{figure*}[t!]
    \centering
    \includegraphics[width=\textwidth]{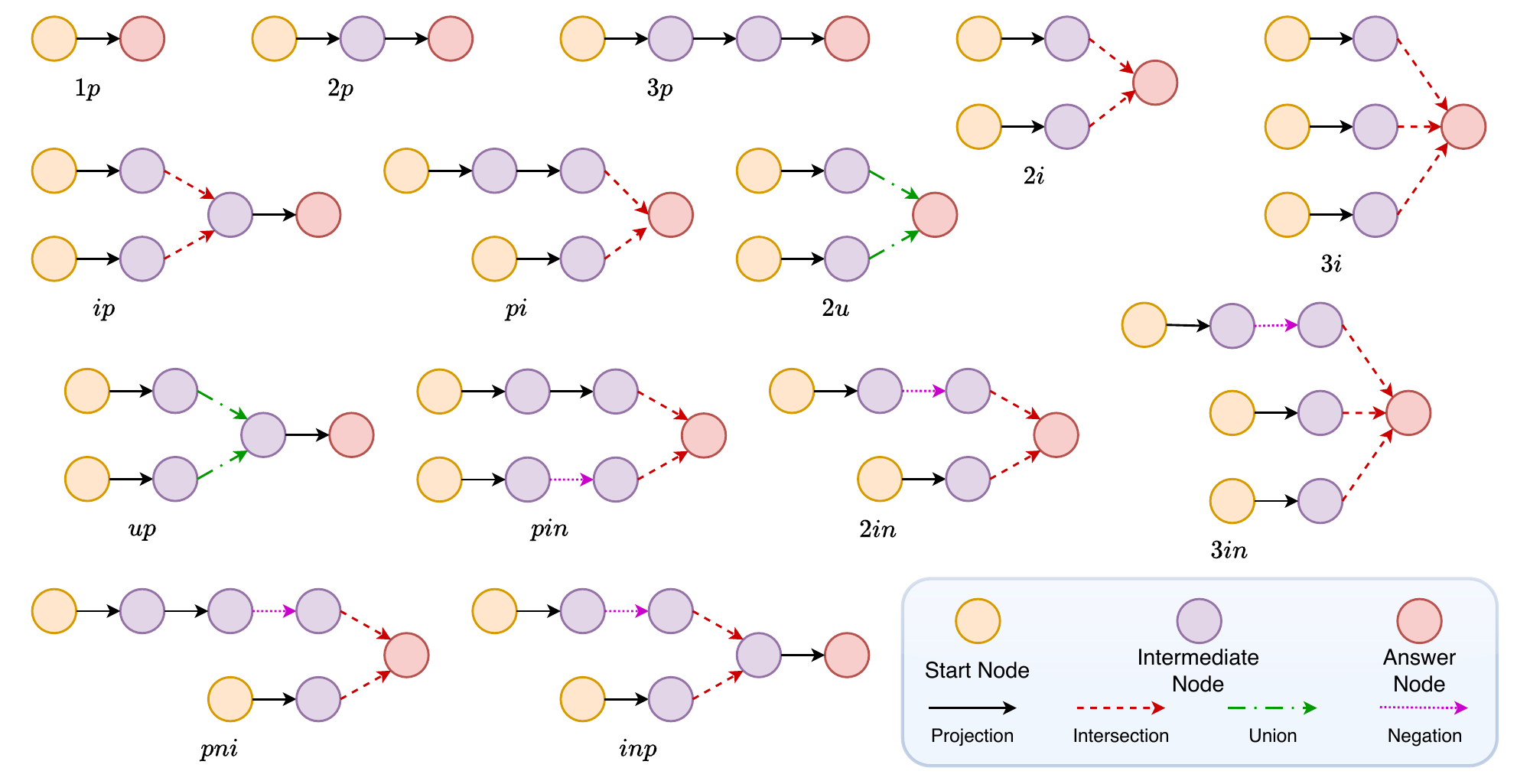}
    \caption{Standard query structures used to generate First Order Logical (FOL) queries. Here, $p$ is relation projection, $i$ is intersection, $n$ is negation, and $u$ is union operator.}
    \label{fig:query_structures}
\end{figure*}

{\setlength{\parindent}{0pt}\textbf{Baselines.}
We use GNN-QE~\cite{galkin2022inductive,zhu2022neural} and NodePiece-QE~\cite{galkin2022inductive,galkin2021nodepiece} as the baselines. GNN-QE uses NBF-Net for relation projection and employs fuzzy set operations for other FOL operators. In NodePiece-QE, a node’s embedding depends on its directly connected relations and its distance from predefined anchor nodes. Its FOL operators are implemented using CQD-Beam~\cite{arakelyan2020complex}. In another variant of NodePiece-QE, the node embeddings are processed through a relational GNN encoder before being passed to CQD-Beam. However, NodePiece-QE cannot handle the negation operator. So we compare it with non-negative queries only.}

We are not including any \emph{transductive} model for comparison, as they perform poorly for \emph{inductive} reasoning, as studied in~\cite{galkin2022inductive}. We use a heuristic method~\cite{galkin2022inductive} as a baseline to evaluate the performance of simpler models in inductive settings.

We exclude ULTRA~\cite{galkin2023towards,galkin2024zero} from our comparison because it addresses a different problem. It is a Foundation Model designed for zero-shot transfer across multiple KGs, while our work optimizes efficiency for a single graph. ULTRA employs dual training, which involves message passing for both relational and entity graphs (effectively training two GNN-QE encoders). This design incurs significant computational overhead, contradicting our primary goal of scalable inference. Additionally, we are already comparing InductWave against GNN-QE, which serves as the backbone of ULTRA’s reasoning module.

{\setlength{\parindent}{0pt}\textbf{Training Graphs.}
In the first dataset, we divide the entities of the FB15k-(237) dataset into three disjoint subsets: $\mathcal{V}_{train}$, $\mathcal{V}_{valid}$, and $\mathcal{V}_{test}$. We create the training graph with $\mathcal{V}_{train} = \tau*\mathcal{V}$ nodes, where $\tau=0.4$. This setup, including the value of $\tau$, is based on~\cite{galkin2022inductive}. The remaining nodes, $(1-\tau)*\mathcal{V}$, are split equally between $\mathcal{V}_{valid}$ and $\mathcal{V}_{test}$. The inference graphs for valid and test cases are constructed from $\mathcal{V}^{inf}=\mathcal{V}_{valid/test} \cup \mathcal{V}_{train}$, with the inference ratio $\mathcal{V}^{inf}/\mathcal{V}_{train}=175\%$. This dataset is intended for testing in the same environment, where GNN-QE was evaluated in~\cite{galkin2022inductive}.}

To further reduce the memory footprint, we create a second set of datasets. In this case, instead of using a train graph consisting all the nodes from $\mathcal{V}_{train}$. We generate a train graph from a subset $\mathcal{V}_{train_0}$ ($\mathcal{V}_{train_0}\subsetneq\mathcal{V}_{train}$) consisting of $l$ nodes, selected randomly. Additionally, we create $k$ random context graphs, each with $l$ nodes, drawn from $\mathcal{V}_{train}$, resulting in entity sets $\mathcal{V}^k_{ctxt} = \{\mathcal{V}_{ctxt_1}, \mathcal{V}_{ctxt_2}, \ldots, \mathcal{V}_{ctxt_{k}}\}$. We generate inference graphs for valid and test cases from  $\mathcal{V}^{inf_0}_{valid/test} = \mathcal{V}_{valid/test} \cup \mathcal{V}_{train_0}$, respectively. We sample multiple sizes for $\mathcal{V}_{train}$ by varying $\tau\in\{0.2, 0.3,\allowbreak \ldots, 0.9\}$ in $\mathcal{V}_{train} = \tau*\mathcal{V}$. This results in a diverse range of graphs with the ratio $\mathcal{V}^{inf}/\mathcal{V}_{train}$ ranging from $106\%$ to $300\%$. Note that we omitted $\tau = 0.1$ because the query generation module~\cite{ren2020beta} was unable to create queries due to the small graph size.

We use Wiki-KG to evaluate the scalability of our model. The KG consists of approximately $2.5$ million nodes and more than $16$ million triples. It has $512$ unique relations. Similar to FB15k-(237), we create a training set that contains $1.5$ million nodes. From this set, we randomly generate a training graph and $k$ context graphs, each of size $l$. Both the valid and test sets include an additional $500k$ nodes, containing $5$ million known and $600k$ missing edges.

{\setlength{\parindent}{0pt}\textbf{Query Generation.}
Queries are generated using the query generation module described in~\cite{ren2020beta}. We utilize 14 different query structures, including \emph{1p}, \emph{2p}, \emph{3p}, \emph{2i}, \emph{3i}, \emph{pi}, \emph{ip}, \emph{2u}, \emph{up}, \emph{pni}, \emph{pin}, \emph{inp}, \emph{2in}, and \emph{3in}, where $p$ is relation projection, $i$ is intersection, $n$ is negation, and $u$ is union (Figure \ref{fig:query_structures}). Detailed statistics on the number of queries generated are available in the Supplementary Material.}

{\setlength{\parindent}{0pt}\textbf{Hyper-Parameters.}
\begin{table*}[t!]
    \centering
    \caption{Hyperparameters of \model{} on different datasets.}
    \begin{adjustbox}{width=\textwidth}
    \begin{tabular}{llccc}
        \toprule
        \multicolumn{2}{l}{\bf{Hyperparameter}} & \multicolumn{1}{>{\centering\arraybackslash}p{3cm}}{\bf{FB15k-(237)} \newline \model{} (3 layers)} & \multicolumn{1}{>{\centering\arraybackslash}p{3cm}}{\bf{FB15k-(237) \newline \model{} (2 layers)}} & \multicolumn{1}{>{\centering\arraybackslash}p{3cm}}{\bf{Wiki-KG \newline \model{} (2 layers)}}\\
        \midrule
        \multirow{4}{*}{\bf{GNN}}
        & \#layers & 3 & 2 & 2\\
        & hidden dim (real) & 16 & 16 & 8 \\
        & hidden dim (imaginary) & 16 & 16 & 8 \\
        & composition & DistMult~\cite{yang2014embedding} & DistMult & DistMult \\
        & aggregation & PNA~\cite{corso2020principal} & PNA & mean \\
        \midrule
        \multirow{2}{*}{\bf{MLP}}
        & \#layer & 2 & 2 & 2 \\
        & hidden dim. & 64 & 64 & 32 \\
        \midrule
        \bf{Traversal Dropout} & probability & 0.45 & 0.35 & 0.35 \\
        \midrule
        \multirow{8}{*}{\bf{Learning}}
        & batch size & 36 & 36 & 6 \\
        & sample weight & uniform across queries & uniform across queries & uniform across queries \\
        & loss & BCE & BCE & BCE \\
        & \# negatives & 32 & 32 & 32 \\
        & optimizer        & Adam & Adam & Adam \\
        & learning rate    & 5e-3 & 5e-3 & 5e-3 \\
        & iterations (\#batch) & 30,000 & 30,000 & 20,000 \\
        & adv. temperature & 0.1 & 0.12 & 0.12 \\
        \midrule
        \multirow{5}{*}{\bf{Graph Wavelet}} 
        & scaling factor & 10 & 10 & 10 \\
        & chebyshev polynomial apprx & 37 & 37 & 5 \\
        & KG laplacian (g) & 0.25 & 0.25 & 0.25 \\
        & Step size for $t_1$ & 4 & 4 & 4 \\
        & Step size for $t_2$ & 3 & 3 & 3 \\
        \bottomrule
    \end{tabular}
    \end{adjustbox}
    \label{tab:hyperparameter}
\end{table*}

In earlier models\cite{ren2020query2box,ren2020beta}, queries were processed in batches, with each batch containing a single type of query. The inability to merge different query types due to varying lengths hindered scalability as the number of query types increased. To address this issue, ~\cite{zhu2022neural} proposed a non-recursive approach to query execution, in which the queries are converted into postfix notation. This allows the use of stacks for query execution. Similarly, we also convert the queries to postfix notation for InductWave.}

For FB15k-(237), we set $l=0.5*|\mathcal{V}_{train}|$ (0.5 is used ensure less memory footprint), and generate two context graphs ($k = 2$), for all ratios except $106\%$. We do the same for Wiki-KG. In the case $\mathcal{V}^{inf}/\mathcal{V}_{train}=106\%$, we take $k=0$. We train \model{} with $\mathcal{V}_{train_0}$ and $k$ context graphs in rotation among epochs. The baseline models were trained on $\mathcal{V}_{train_0}$. The rationale behind generating $k$ context graphs is that, at any given time, only one graph (amongst train and context graphs) will be in the memory. This graph would be smaller than $|\mathcal{V}_{train}|$ (saving memory), while still providing additional context. Concurrently, we experimented with the dataset containing no context graphs (first dataset), where both \model{} and baselines are trained on $\mathcal{V}_{train}$ only.

We train \msgpass{} with two-layer and three-layer configurations. NBF-Net in GNN-QE is trained with four message-passing layers, as in~\cite{galkin2022inductive}. We do not train our models for negative query structures in Wiki-KG because generating negative queries requires adding reverse relations to the KG. This would effectively double the number of triples in the graph, which are already numerous. As in the baseline~\cite{galkin2022inductive}, we are not conducting experiments for GNN-QE on the Wiki-KG dataset due to GPU memory constraints. 

For the Graph Wavelet parameters, an ablation study was deemed unsuitable, so we performed a grid search to determine the optimal values. The parameters explored included the scaling factor $= [0.1, 1, 10, 50]$, $g = [0.1, 0.15, 0.2, 0.25]$, and $t_1, t_2 = [2, 3, 4, 5, 10]$. We used the chebyshev approximation parameter consistent with the original work~\cite{donnat2018learning}. For WikiKG, this parameter of chebyshev approximation was chosen specifically to reduce computational costs due to its large size. For traversal dropout, we tested values in $[0.2, 0.25, 0.3, 0.35, 0.4, 0.45, 0.5]$. All experiments for \model{} were conducted on 40GB Nvidia A100 GPUs. Table \ref{tab:hyperparameter} presents the hyperparameters for \model{}. 

\subsection{Result Analysis}
\begin{table*}[t!]
    \centering
    \caption{HITS@10 (\%) score for \emph{inductive} query answering, where all the nodes in $\mathcal{V}_{train}$ are used in a single training graph, $\mathcal{V}^{inf}/\mathcal{V}_{train} = 175\%$. avg$_p$ is the average of positive queries, and avg is the average of all the queries.}
    \begin{adjustbox}{width=\textwidth}
    \begin{tabular}{lrrrrrrrrrrrrrrrr}
        \toprule
        \bf{Model} & \bf{avg} & \bf{avg$_p$} & \bf{1p} & \bf{2p} & \bf{3p} & \bf{2i} & \bf{3i} & \bf{pi} & \bf{ip} & \bf{2u} & \bf{up} & \bf{2in} & \bf{3in} & \bf{inp} & \bf{pin} & \bf{pni} \\
        \midrule
            Edge-type Heuristic & 8.0 & 10.1 & 17.6 & 8.2 & 9.9 & 10.6 & 13.0 & 9.6 & 8.2 & 5.3 & 8.5 & 2.6 & 3.1 & 8.7 & 3.7 & 2.8 \\
            NodePiece-QE & - & 10.9 & 25.2 & 8.3 & 7.9 & 11.8 & 13.0 & 9.2 & 8.6 & 7.0 & 6.8 & - & - & - & - & - \\ 
            NodePiece-QE w/ GNN & - & 26.5 & 42.2 & 18.1 & 11.2 & 35.8 & 44.2 & 26.3 & 21.6 & 24.7 & 14.6 & - & - & - & - & - \\ 
            GNN-QE & 44.8 & 50.9 & 62.7 & \textbf{37.1} & \textbf{32.2} & 72.7 & 83.5 & \textbf{57.6} & \textbf{43.4} & 38.3 & \textbf{30.3} & 32.7 & 46.0 & \textbf{29.4} & \textbf{27.2} & \textbf{33.4} \\
            \model{} (3 layers) &  \textbf{45.5} & \textbf{51.6} & \textbf{67.2} & 36.4 & 31.3 & \textbf{74.5} & \textbf{84.9} & \textbf{57.6} & 42.3 & \textbf{41.3} & 28.6 & \textbf{34.6} & \textbf{49.0} & \textbf{29.4} & 26.5 & 33.2 \\
            \model{} (2 layers) & 41.7 & 47.9 & 63.1 & 33.5 & 29.2 & 70.1 & 81.3 & 54.1 & 38.8 & 33.3 & 27.3 & 29.2 & 43.5 & 28.3 & 24.7 & 26.6 \\
        \bottomrule
    \end{tabular}
    \end{adjustbox}
    \label{tab:hits175worandtrain}
\end{table*}
\begin{table*}[t!]
    \centering
    \caption{HITS@10 (\%) score for \emph{inductive} query answering for $\mathcal{V}^{inf}/\mathcal{V}_{train} = 175\%$. avg$_p$ is the average of positive queries, and avg is the average of all the queries.}
    \begin{adjustbox}{width=\textwidth}
    \begin{tabular}{lrrrrrrrrrrrrrrrr}
        \toprule
        \bf{Model} & \bf{avg} & \bf{avg$_p$} & \bf{1p} & \bf{2p} & \bf{3p} & \bf{2i} & \bf{3i} & \bf{pi} & \bf{ip} & \bf{2u} & \bf{up} & \bf{2in} & \bf{3in} & \bf{inp} & \bf{pin} & \bf{pni} \\
        \midrule
            Edge-type Heuristic & 8.4 & 10.5 & 18.8 & 9.9 & 10.7 & 10.4 & 13.0 & 9.5 & 8.3 & 5.0 & 8.8 & 2.5 & 2.9 & 10.8 & 4.1 & 2.7 \\
            NodePiece-QE & - & 10.9 & 24.6 & 9.9 & 8.3 & 11.5 & 12.5 & 9.5 & 8.5 & 6.0 & 7.4 & - & - & - & - & - \\ 
            NodePiece-QE w/ GNN & - & 28.4 & 46.3 & 19.6 & 12.6 & 36.3 & 43.1 & 25.5 & 23.7 & 30.2 & 18.3 & - & - & - & - & - \\
            GNN-QE & 41.4 & 48.1 & 61.4 & 37.1 & 29.2 & 70.5 & 81.4 & 55.5 & 39.2 & 33.7 & 24.8 & 28.4 & 40.2 & 28.8 & 23.4 & 25.6 \\
            \model{} (3 layers) & \textbf{43.5} & \textbf{50.5} & \textbf{65.7} & \textbf{39.5} & \textbf{32.2} & \textbf{72.4} & \textbf{82.7} & \textbf{57.0} & \textbf{43.2} & \textbf{34.5} & \textbf{27.8} & \textbf{29.2} & \textbf{41.6} & \textbf{30.9} & \textbf{26.4} & \textbf{26.6} \\
            \model{} (2 layers) & 41.5 & 48.6 & 64.2 & 38.8 & 31.0 & 71.0 & 81.5 & 56.1 & 39.5 & 30.4 & 24.9 & 25.2 & 40.7 & 30.8 & 24.8 & 22.8 \\
        \bottomrule
    \end{tabular}
    \end{adjustbox}
    \label{tab:hits175}
\end{table*}
Initially, we test our model on the first dataset generated from FB15k-(237). This dataset consists of one training graph induced from all the nodes in $\mathcal{V}_{train}$, without any context graphs, with a ratio of $\mathcal{V}^{inf}/\mathcal{V}_{train} = 175\%$. The HITS@10 results for this configuration are in Table~\ref{tab:hits175worandtrain}. The average results indicate that \model{}, with a 3-layer \msgpass{}, outperforms GNN-QE, while the 2-layer model performs slightly worse. We believe that integrating wavelet embeddings enables \msgpass{} to more effectively capture the graph structure, which facilitates convergence with fewer message passing layers. Notably, the 2-layer version performs marginally inferior than GNN-QE, despite having half the layers. This makes it a suitable choice in scenarios with resource constraints. NodePiece-QE with GNN performs better than without GNN. However, it is still worse than GNN-QE and \model{} for positive queries (queries without a negation operator) (avg$_p$). As expected, the Edge-type heuristic model performs the worst among all, as it relies on simple heuristics for query answering without any training.

Since \model{} performs better in the previous scenario. We conduct additional experiments by generating small training graphs of size $l$, along with context graphs (the second set of datasets, generated from FB15k-(237) by varying $\tau$). This is to further reduce the memory overhead for training. Table \ref{tab:hits175} presents the HITS@10 scores on FB15k-(237), where $\mathcal{V}^{inf}/\mathcal{V}_{train} = 175\%$. Here, \model{} (3 layers) consistently outperforms other models across all query types. Meanwhile, \model{} (2 layers) performs on par with GNN-QE. Both versions of NodePiece-QE and the heuristic model exhibit similar trends as in Table~\ref{tab:hits175worandtrain}.

\begin{table*}[t!]
\caption{HITS@10 (\%) score for \emph{inductive} query answering for all $\mathcal{V}^{inf}/\mathcal{V}_{train}$ ratios. avg$_p$ is the average of positive queries, and avg is the average of all the queries.}
\label{tab:hits10_all}
\begin{adjustbox}{width=\textwidth}
\begin{tabular}{@{}clrrrrrrrrrrrrrrrr@{}}
\toprule
\multicolumn{1}{l}{Ratio} & Model & \bf{avg} & \bf{avg$_p$} & 1p & 2p & 3p & 2i & 3i & pi & ip & 2u & up & 2in & 3in & inp & pin & pni \\ \midrule
\multirow{6}{*}{300\%} & Edge-type Heuristic & 9.5 & 11.8 & 19.3 & 11.3 & 11.6 & 12.2 & 15.2 & 11.5 & 9.5 & 6.0 & 9.5 & 3.3 & 3.7 & 11.4 & 4.9 & 3.8\\
 & NodePiece-QE & - & 9.2 & 19.8 & 7.5 & 5.0 & 11.3 & 12.5 & 9.2 & 7.1 & 5.9 & 4.9 & - & - & - & - & - \\ 
 & NodePiece-QE w/ GNN & - & 17.8 & 35.6 & 14.6 & 8.8 & 22.7 & 20.2 & 16.4 & 13.7 & 16.1 & 11.9 & - & - & - & - & -  \\ 
 & GNN-QE & 36.4 & 43.8 & 53.2 & 33.6 & 26.0 & 66.6 & 80.9 & 53.9 & 36.1 & 23.3 & 20.6 & 19.0 & 36.3 & 24.7 & 19.7 & 15.6 \\
 & \model{} (3 layers) & 38.0 & 45.3 & 54.4 & \textbf{36.2} & 27.9 & 68.0 & 81.5 & 54.7 & 37.1 & 25.1 & \textbf{22.4} & \textbf{20.9} & 38.3 & \textbf{26.7} & \textbf{21.3} & \textbf{17.8} \\
  & \model{} (2 layers) & \textbf{38.2} & \textbf{45.6} & \textbf{54.5} & 35.6 & \textbf{28.2} & \textbf{69.1} & \textbf{82.2} & \textbf{55.3} & \textbf{38.5} & \textbf{25.4} & 21.6 & 20.4 & \textbf{39.7} & 26.5 & \textbf{21.3} & 16.5 \\\midrule
\multirow{6}{*}{217\%} & Edge-type Heuristic & 9.1 & 11.5 & 19.2 & 11.0 & 11.5 & 11.7 & 14.6 & 10.9 & 9.2 & 5.7 & 9.5 & 2.9 & 3.3 & 10.9 & 4.4 & 3.1 \\
 & NodePiece-QE & - & 11.1 & 22.6 & 10.3 & 9.1 & 11.6 & 12.8 & 9.9 & 9.0 & 6.4 & 7.9 & - & - & - & - & - \\ 
 & NodePiece-QE w/ GNN & - & 22.6 & 42.1 & 15.4 & 9.2 & 28.9 & 33.8 & 19.9 & 17.9 & 21.3 & 14.4 & - & - & - & - & - \\ 
 & GNN-QE & 37.8 & 44.7 & 53.3 & 35.5 & 28.3 & 67.4 & 81.3 & 54.1 & 35.7 & 24.0 & 22.9 & 22.3 & \textbf{39.0} & 28.1 & 19.8 & 17.5\\
 & \model{} (3 layers) & \textbf{39.2} & \textbf{46.8} & \textbf{60.3} & 37.0 & 29.6 & \textbf{69.4} & \textbf{82.1} & \textbf{55.2} & 37.5 & \textbf{26.9} & \textbf{23.6} & \textbf{22.8} & 38.1 & 28.3 & 20.0 & \textbf{18.4} \\
  & \model{} (2 layers) & \textbf{39.2} & 46.7 & 58.7 & \textbf{38.4} & \textbf{30.2} & 68.6 & 81.4 & 53.5 & \textbf{39.2} & 26.6 & \textbf{23.6} & 21.3 & 38.3 & \textbf{29.2} & \textbf{21.9} & 17.7 \\\midrule
\multirow{6}{*}{175\%} & Edge-type Heuristic & 8.4 & 10.5 & 18.8 & 9.9 & 10.7 & 10.4 & 13.0 & 9.5 & 8.3 & 5.0 & 8.8 & 2.5 & 2.9 & 10.8 & 4.1 & 2.7 \\
 & NodePiece-QE & - & 10.9 & 24.6 & 9.9 & 8.3 & 11.5 & 12.5 & 9.5 & 8.5 & 6.0 & 7.4 & - & - & - & - & -\\ 
 & NodePiece-QE w/ GNN & - & 28.4 & 46.3 & 19.6 & 12.6 & 36.3 & 43.1 & 25.5 & 23.7 & 30.2 & 18.3 & - & - & - & - & -\\ 
 & GNN-QE & 41.4 & 48.1 & 61.4 & 37.1 & 29.2 & 70.5 & 81.4 & 55.5 & 39.2 & 33.7 & 24.8 & 28.4 & 40.2 & 28.8 & 23.4 & 25.6 \\
 & \model{} (3 layers) & \textbf{43.5} & \textbf{50.5} & \textbf{65.7} & \textbf{39.5} & \textbf{32.2} & \textbf{72.4} & \textbf{82.7} & \textbf{57.0} & \textbf{43.2} & \textbf{34.5} & \textbf{27.8} & \textbf{29.2} & \textbf{41.6} & \textbf{30.9} & \textbf{26.4} & \textbf{26.6} \\
  & \model{} (2 layers) & 41.5 & 48.6 & 64.2 & 38.8 & 31.0 & 71.0 & 81.5 & 56.1 & 39.5 & 30.4 & 24.9 & 25.2 & 40.7 & 30.8 & 24.8 & 22.8 \\\midrule
\multirow{6}{*}{150\%} & Edge-type Heuristic & 7.9 & 10.0 & 18.9 & 8.9 & 9.4 & 10.4 & 13.2 & 8.8 & 8.2 & 4.6 & 7.7 & 2.3 & 3.0 & 9.3 & 3.5 & 2.1 \\
 & NodePiece-QE & - & 10.6 & 23.5 & 9.2 & 7.3 & 11.5 & 13.1 & 9.7 & 8.6 & 5.8 & 6.8 & - & - & - & - & -\\ 
 & NodePiece-QE w/ GNN & - & 36.8 & 55.3 & 25.8 & 12.2 & 49.7 & 59.8 & 37.3 & 32.6 & 35.8 & 22.4 & - & - & - & - & -\\ 
 & GNN-QE & 40.5 & 48.6 & \textbf{66.8} & 36.2 & 28.0 & 72.4 & 83.5 & 55.5 & 39.5 & 31.6 & 23.5 & 23.3 & 39.8 & 27.8 & 18.2 & 20.9 \\
 & \model{} (3 layers) & \textbf{42.2} & \textbf{49.7} & 66.2 & \textbf{38.3} & \textbf{30.5} & \textbf{73.8} & \textbf{84.6} & \textbf{56.0} & \textbf{40.1} & \textbf{32.8} & \textbf{25.2} & \textbf{25.2} & \textbf{42.1} & \textbf{29.7} & \textbf{21.1} & \textbf{24.8} \\
  & \model{} (2 layers) & 40.6 & 47.8 & 63.2 & 37.1 & 30.2 & 70.8 & 82.7 & 53.7 & 40.0 & 28.4 & 24.5 & 23.8 & 40.1 & 29.8 & 21.7 & 22.1 \\\midrule
\multirow{6}{*}{133\%} & Edge-type Heuristic & 7.4 & 9.3 & 16.7 & 8.7 & 8.7 & 9.9 & 12.6 & 8.5 & 7.1 & 4.5 & 7.3 & 2.4 & 2.8 & 8.3 & 3.9 & 2.3 \\
 & NodePiece-QE & - & 10.8 & 23.1 & 10.0 & 8.3 & 11.1 & 13.0 & 9.7 & 8.8 & 5.8 & 7.0 & - & - & - & - & -\\ 
 & NodePiece-QE w/ GNN & - & 42.1 & 60.7 & 30.2 & 13.4 & 56.0 & 66.9 & 43.8 & 39.7 & \textbf{42.0} & 26.3 & - & - & - & - & -\\ 
 & GNN-QE & 44.0 & 50.7 & 68.9 & 38.4 & 30.7 & 75.1 & 86.1 & 55.0 & 38.8 & 36.8 & 26.2 & 29.4 & 45.3 & 29.3 & 25.3 & 30.5 \\
 & \model{} (3 layers) & \textbf{45.2} & \textbf{51.6} & \textbf{69.1} & \textbf{39.8} & \textbf{31.8} & \textbf{75.6} & \textbf{86.6} & \textbf{56.6} & \textbf{40.3} & 36.8 & \textbf{27.5} & \textbf{32.0} & \textbf{47.5} & \textbf{30.2} & \textbf{28.6} & \textbf{30.7} \\
  & \model{} (2 layers) & 40.7 & 47.5 & 65.2 & 37.1 & 30.5 & 70.7 & 82.6 & 51.3 & 35.9 & 30.6 & 24.0 & 25.2 & 41.3 & 27.9 & 23.5 & 24.0 \\\midrule
\multirow{6}{*}{121\%} & Edge-type Heuristic & 7.0 & 8.9 & 19.8 & 7.4 & 8.5 & 9.0 & 11.3 & 7.4 & 5.9 & 3.7 & 6.9 & 2.3 & 2.7 & 7.1 & 3.3 & 2.1 \\
 & NodePiece-QE & - & 9.6 & 24.9 & 8.1 & 7.7 & 9.8 & 11.2 & 7.4 & 6.5 & 4.5 & 6.1 & - & - & - & - & -\\ 
 & NodePiece-QE w/ GNN & - & 42.3 & 61.7 & 28.4 & 11.7 & 58.7 & 68.1 & 40.0 & \textbf{38.9} & \textbf{46.4} & \textbf{26.4} & - & - & - & - & -\\ 
 & GNN-QE & 41.9 & \textbf{49.0} & 70.4 & \textbf{34.7} & \textbf{29.0} & 73.2 & 81.7 & 52.4 & 37.2 & 36.8 & 25.7 & 30.1 & 39.1 & 26.0 & \textbf{21.7} & 28.9 \\
 & \model{} (3 layers) & \textbf{42.2} & 48.7 & \textbf{71.5} & 33.8 & 27.1 & \textbf{73.9} & \textbf{82.1} & \textbf{52.6} & 35.5 & 37.8 & 23.4 & \textbf{31.6} & \textbf{41.7} & \textbf{28.0} & 20.7 & \textbf{31.0} \\
  & \model{} (2 layers) & 38.9 & 45.4 & 66.2 & 33.7 & 27.7 & 67.7 & 77.1 & 49.9 & 33.6 & 31.6 & 20.8 & 27.4 & 38.0 & 25.9 & 20.2 & 25.6 \\\midrule
\multirow{6}{*}{113\%} & Edge-type Heuristic & 
6.0 & 7.7 & 18.5 & 6.2 & 6.4 & 7.2 & 9.8 & 6.5 & 5.3 & 3.3 & 6.0 & 1.9 & 1.9 & 5.8 & 3.2 & 1.8 \\
 & NodePiece-QE & - & 8.7 & 23.0 & 7.0 & 5.9 & 8.1 & 10.4 & 7.5 & 6.4 & 4.4 & 5.3 & - & - & - & - & -\\ 
 & NodePiece-QE w/ GNN & - & 45.1 & 61.9 & 28.4 & 12.1 & 60.7 & 68.8 & 44.1 & \textbf{44.9} & \textbf{56.0} & \textbf{28.8} & - & - & - & - & -\\ 
 & GNN-QE & 40.7 & 46.9 & \textbf{70.9} & 31.9 & 26.1 & 66.4 & 76.3 & 50.6 & 36.7 & 40.0 & 23.4 & \textbf{31.0} & 38.3 & 24.7 & 20.4 & \textbf{33.4}\\
 & \model{} (3 layers) & \textbf{44.2} & \textbf{50.7} & 69.0 & \textbf{38.8} & \textbf{30.5} & \textbf{75.8} & \textbf{86.4} & \textbf{55.7} & 39.2 & 35.8 & 25.1 & 30.9 & \textbf{46.1} & \textbf{29.5} & \textbf{24.8} & 31.0 \\
  & \model{} (2 layers) & 38.2 & 44.2 & 66.6 & 31.5 & 26.2 & 61.7 & 72.3 & 47.6 & 34.6 & 34.6 & 22.8 & 27.3 & 36.4 & 24.7 & 21.0 & 27.3 \\\midrule
\multirow{6}{*}{106\%} & Edge-type Heuristic & 5.8 & 7.2 & 16.5 & 6.3 & 4.9 & 8.0 & 10.9 & 5.6 & 4.8 & 3.1 & 5.0 & 1.7 & 1.9 & 5.7 & 3.1 & 3.3  \\
 & NodePiece-QE & - & 8.8 & 24.2 & 7.2 & 4.1 & 9.1 & 13.3 & 5.9 & 6.7 & 3.9 & 4.7 & - & - & - & - & - \\ 
 & NodePiece-QE w/ GNN & - & \textbf{49.8} & 66.8 & 33.9 & 12.6 & \textbf{67.1} & \textbf{73.1} & \textbf{46.1} & \textbf{54.2} & \textbf{62.7} & \textbf{31.8} & - & - & - & - & - \\ 
 & GNN-QE & 40.5 & 46.5 & 69.6 & 35.3 & 27.8 & 66.1 & 71.8 & 45.1 & 34.9 & 40.9 & 27.0 & 32.3 & 34.6 & 23.0 & 22.2 & 36.5 \\
 & \model{} (3 layers) & \textbf{41.7} & 47.4 & \textbf{69.3} & \textbf{35.5} & \textbf{29.4} & \textbf{67.1} & 72.4 & 45.7 & 37.1 & 41.9 & 27.9 & \textbf{34.4} & \textbf{35.4} & \textbf{24.3} & 22.8 & \textbf{40.4} \\
  & \model{} (2 layers) & 37.6 & 43.1 & 63.3 & 33.7 & 27.6 & 60.9 & 69.2 & 41.5 & 32.5 & 32.8 & 26.1 & 25.2 & 33.0 & 21.9 & \textbf{28.0} & 30.1 \\ \bottomrule
\end{tabular}
\end{adjustbox}
\end{table*}

\begin{table*}[!t]
\caption{HITS@3 (\%) score for \emph{inductive} query answering for all $\mathcal{V}^{inf}/\mathcal{V}_{train}$ ratios. avg$_p$ is the average of positive queries, and avg is the average of all the queries.}
\label{tab:hits3_all}
\begin{adjustbox}{width=\textwidth}
\begin{tabular}{@{}clrrrrrrrrrrrrrrrr@{}}
\toprule
\multicolumn{1}{l}{Ratio} & Model & \bf{avg} & \bf{avg$_p$} & 1p & 2p & 3p & 2i & 3i & pi & ip & 2u & up & 2in & 3in & inp & pin & pni \\ \midrule
\multirow{6}{*}{300\%} & Edge-type Heuristic & 4.3 & 5.3 & 9.5 & 5.1 & 5.2 & 5.2 & 6.7 & 4.7 & 4.4 & 2.5 & 4.4 & 1.5 & 1.6 & 4.9 & 2.4 & 1.6 \\
 & NodePiece-QE & - & 4.8 & 12.8 & 3.8 & 2.2 & 5.4 & 5.9 & 4.2 & 3.7 & 3.1 & 2.3 & - & - & - & - & -\\ 
 & NodePiece-QE w/ GNN & - & 9.9 & 21.5 & 7.8 & 4.4 & 12.6 & 11.7 & 8.8 & 7.7 & 8.6 & 6.0 & - & - & - & - & -\\ 
 & GNN-QE & 26.6 & 33.3 & 43.2 & 22.6 & 16.1 & 52.3 & 68.9 & 41.5 & 26.3 & 16.0 & 12.4 & 12.7 & 24.1 & 14.8 & 11.6 & 9.9 \\
 & \model{} (3 layers) & 27.6 & 34.3 & 43.0 & \textbf{24.8} & \textbf{17.9} & 53.6 & 69.5 & 42.2 & 27.1 & \textbf{16.8} & \textbf{13.7} & 13.1 & 24.7 & \textbf{16.3} & \textbf{13.0} & \textbf{11.0} \\
  & \model{} (2 layers) & \textbf{27.8} & \textbf{34.5} & \textbf{44.1} & 24.0 & 17.8 & \textbf{54.0} & \textbf{70.1} & \textbf{42.3} & \textbf{28.6} & \textbf{16.8} & 12.8 & \textbf{13.4} & \textbf{26.2} & 16.1 & 12.5 & 10.6 \\\midrule
\multirow{6}{*}{217\%} & Edge-type Heuristic & 4.2 & 5.2 & 9.6 & 5.0 & 5.2 & 5.1 & 6.2 & 4.5 & 4.4 & 2.4 & 4.6 & 1.4 & 1.5 & 4.9 & 2.2 & 1.3 \\
 & NodePiece-QE & - & 5.4 & 14.3 & 4.8 & 4.3 & 5.0 & 5.5 & 4.2 & 4.3 & 2.9 & 3.7 & - & - & - & - & - \\ 
 & NodePiece-QE w/ GNN & - & 12.6 & 26.1 & 8.2 & 4.7 & 16.0 & 19.5 & 10.2 & 10.0 & 11.7 & 7.6 & - & - & - & - & -\\ 
 & GNN-QE & 27.7 & 34.4 & 44.1 & 24.4 & 18.4 & 54.2 & 70.4 & 41.7 & 26.5 & 15.6 & 14.3 & 13.2 & \textbf{26.0} & 17.3 & 12.1 & 10.3 \\
 & \model{} (3 layers) & \textbf{28.7} & \textbf{35.8} & \textbf{47.9} & 25.6 & 19.6 & \textbf{55.7} & \textbf{71.0} & \textbf{42.3} & 28.0 & \textbf{17.4} & \textbf{14.9} & \textbf{13.4} & 25.1 & 17.3 & 12.3 & \textbf{10.9} \\
  & \model{} (2 layers) & 28.6 & 35.6 & 47.5 & \textbf{26.3} & \textbf{19.7} & 54.9 & 70.2 & 40.9 & \textbf{29.1} & 17.3 & 14.7 & \textbf{13.4} & 25.4 & \textbf{17.5} & \textbf{13.0} & \textbf{10.9} \\\midrule
\multirow{6}{*}{175\%} & Edge-type Heuristic & 3.8 & 4.8 & 9.8 & 4.6 & 4.9 & 4.4 & 5.6 & 3.9 & 3.9 & 2.1 & 4.3 & 1.1 & 1.3 & 4.7 & 2.0 & 1.2 \\
 & NodePiece-QE & - & 5.6 & 16.6 & 4.9 & 3.9 & 5.1 & 5.5 & 4.1 & 4.2 & 2.6 & 3.5 & - & - & - & - & -\\ 
 & NodePiece-QE w/ GNN & - & 16.5 & 30.9 & 10.7 & 6.5 & 21.2 & 25.9 & 14.0 & 13.4 & 16.4 & 9.6 & - & - & - & - & -\\ 
 & GNN-QE & 30.2 & 36.9 & 50.2 & 25.5 & 18.9 & 56.5 & 70.2 & 42.8 & 29.5 & \textbf{23.0} & 15.4 & \textbf{17.8} & 27.0 & 17.8 & 14.0 & 14.8 \\
 & \model{} (3 layers) & \textbf{31.6} & \textbf{38.6} & \textbf{53.1} & \textbf{27.1} & \textbf{21.3} & \textbf{58.5} & \textbf{71.7} & \textbf{43.9} & \textbf{31.6} & 22.2 & \textbf{17.7} & 17.2 & \textbf{28.1} & \textbf{19.0} & \textbf{16.2} & \textbf{15.1} \\
  & \model{} (2 layers) & 30.4 & 37.2 & 51.9 & 26.7 & 20.5 & 56.7 & 70.1 & 43.0 & 30.0 & 20.1 & 16.2 & 15.7 & 27.6 & 18.8 & 15.3 & 13.2 \\\midrule
\multirow{6}{*}{150\%} & Edge-type Heuristic & 3.6 & 4.6 & 10.2 & 4.1 & 4.2 & 4.3 & 5.5 & 3.8 & 4.2 & 1.9 & 3.6 & 1.0 & 1.3 & 3.7 & 1.7 & 0.9 \\
 & NodePiece-QE & - & 5.4 & 14.2 & 4.7 & 3.4 & 5.1 & 5.6 & 4.4 & 4.8 & 2.8 & 3.3 & - & - & - & - & -\\ 
 & NodePiece-QE w/ GNN & - & 24.1 & 38.3 & 16.2 & 6.5 & 32.2 & 40.9 & 23.3 & 21.9 & \textbf{22.8} & 14.3 & - & - & - & - & -\\ 
 & GNN-QE & 29.6 & 37.1 & \textbf{54.6} & 24.3 & 18.0 & 58.0 & 72.4 & 42.3 & 28.9 & 20.3 & 15.0 & 13.6 & 25.4 & 16.7 & 11.1 & 12.9 \\
 & \model{} (3 layers) & \textbf{30.6} & \textbf{38.0} & 53.8 & \textbf{25.7} & \textbf{19.8} & \textbf{59.2} & \textbf{73.4} & \textbf{42.7} & \textbf{29.9} & 21.7 & \textbf{15.8} & \textbf{15.4} & \textbf{27.6} & 17.8 & 12.5 & \textbf{13.5} \\
  & \model{} (2 layers) & 29.6 & 36.7 & 52.0 & 25.1 & 19.6 & 56.3 & 71.5 & 40.9 & 30.0 & 19.0 & 15.5 & 15.2 & 26.1 & \textbf{18.0} & \textbf{13.0} & 12.4 \\\midrule
\multirow{6}{*}{133\%} & Edge-type Heuristic & 3.3 & 4.3 & 8.7 & 4.1 & 3.8 & 4.1 & 5.3 & 3.5 & 3.5 & 1.8 & 3.4 & 1.0 & 1.2 & 3.3 & 1.9 & 1.0 \\
 & NodePiece-QE & - & 5.1 & 12.3 & 5.0 & 4.0 & 4.9 & 5.6 & 4.1 & 4.6 & 2.7 & 3.2 & - & - & - & - & -\\ 
 & NodePiece-QE w/ GNN & - & 29.2 & 43.5 & 20.1 & 7.7 & 38.0 & 48.3 & 29.5 & 28.8 & \textbf{28.8} & \textbf{18.2} & - & - & - & - & - \\ 
 & GNN-QE & 32.3 & 38.8 & 55.3 & 26.1 & 20.2 & 60.3 & 74.6 & 41.6 & 28.8 & 25.9 & 16.8 & 18.9 & 31.1 & 17.7 & 15.9 & \textbf{19.4} \\
 & \model{} (3 layers) & \textbf{33.1} & \textbf{39.4} & \textbf{56.7} & \textbf{26.4} & \textbf{20.9} & \textbf{61.0} & \textbf{75.2} & \textbf{42.6} & \textbf{29.3} & 25.3 & 17.4 & \textbf{21.4} & \textbf{33.2} & \textbf{18.8} & \textbf{16.2} & 18.8 \\
  & \model{} (2 layers) & 29.3 & 35.9 & 51.5 & 25.3 & 19.8 & 55.5 & 70.6 & 38.2 & 26.3 & 20.2 & 15.6 & 15.3 & 28.2 & 16.6 & 13.9 & 13.0 \\\midrule
\multirow{6}{*}{121\%} & Edge-type Heuristic & 3.2 & 4.2 & 11.3 & 3.4 & 4.1 & 3.8 & 4.9 & 3.0 & 2.8 & 1.5 & 3.3 & 1.0 & 1.1 & 2.8 & 1.5 & 0.7 \\
 & NodePiece-QE & - & 4.3 & 11.6 & 3.7 & 3.4 & 4.3 & 4.8 & 3.1 & 3.1 & 2.1 & 2.7 & - & - & - & - & -\\ 
 & NodePiece-QE w/ GNN & - & 30.5 & 46.7 & 19.2 & 7.0 & 41.7 & 50.7 & 26.1 & 29.0 & 34.6 & 19.2 & - & - & - & - & - \\ 
 & GNN-QE & \textbf{31.0} & \textbf{37.6} & 58.4 & \textbf{23.1} & \textbf{18.8} & 58.1 & 69.6 & \textbf{39.8} & \textbf{27.6} & \textbf{26.5} & \textbf{16.3} & \textbf{20.7} & 26.7 & 16.1 & \textbf{13.8} & 18.6 \\
 & \model{} (3 layers) & 30.6 & 36.8 & \textbf{58.9} & 21.8 & 17.2 & \textbf{58.7} & \textbf{70.1} & 39.3 & 25.1 & 26.2 & 14.3 & \textbf{20.7} & \textbf{28.2} & \textbf{17.0} & 12.4 & \textbf{19.1} \\
  & \model{} (2 layers) & 28.1 & 34.0 & 54.4 & \textbf{21.8} & 17.6 & 52.9 & 65.1 & 37.4 & 23.5 & 21.1 & 12.6 & 17.5 & 25.8 & 15.7 & 11.7 & 15.6 \\\midrule
\multirow{6}{*}{113\%} & Edge-type Heuristic & 2.7 & 3.6 & 10.4 & 2.9 & 2.9 & 2.8 & 3.8 & 2.7 & 2.4 & 1.3 & 2.9 & 0.7 & 0.8 & 2.3 & 1.4 & 0.7 \\
 & NodePiece-QE & - & 3.9 & 10.7 & 3.4 & 2.8 & 3.2 & 4.0 & 3.0 & 3.2 & 2.1 & 2.5 & - & - & - & - & -\\ 
 & NodePiece-QE w/ GNN & - & 32.9 & 46.8 & 20.2 & 7.8 & 44.7 & 52.0 & 31.2 & \textbf{33.1} & 41.0 & 19.9 & - & - & - & - & -\\ 
 & GNN-QE & 29.6 & 35.6 & 58.3 & 21.4 & 16.6 & 51.8 & 63.5 & 37.3 & 26.8 & \textbf{29.4} & 14.9 & \textbf{20.7} & 24.9 & 14.8 & 13.0 & \textbf{21.6} \\
 & \model{} (3 layers) & \textbf{32.1} & \textbf{38.6} & \textbf{55.2} & \textbf{25.8} & \textbf{19.6} & \textbf{60.8} & \textbf{75.2} & \textbf{41.9} & 28.8 & 24.4 & \textbf{15.6} & 19.6 & \textbf{32.3} & \textbf{18.1} & \textbf{14.7} & 17.6 \\
  & \model{} (2 layers) & 27.6 & 33.1 & 54.6 & 20.6 & 16.4 & 47.0 & 59.3 & 35.1 & 24.9 & 25.0 & 14.7 & 18.7 & 23.8 & 14.6 & 13.5 & 18.4 \\\midrule
\multirow{6}{*}{106\%} & Edge-type Heuristic & 2.6 & 3.3 & 10.0 & 2.9 & 2.0 & 3.1 & 4.2 & 2.1 & 2.0 & 1.3 & 2.2 & 0.5 & 0.8 & 1.9 & 1.5 & 1.2 \\
 & NodePiece-QE & - & 4.7 & 17.3 & 3.2 & 1.7 & 4.5 & 6.6 & 2.2 & 3.2 & 1.4 & 1.9 & - & - & - & - & -\\ 
 & NodePiece-QE w/ GNN & - & \textbf{38.1} & 52.7 & \textbf{24.7} & 8.5 & \textbf{52.2} & 58.1 & 33.0 & \textbf{41.0} & \textbf{48.8} & \textbf{23.7} & - & - & - & - & -\\ 
 & GNN-QE & 29.6 & 35.2 & \textbf{57.3} & 23.8 & \textbf{17.7} & 51.3 & 59.8 & 34.0 & 24.7 & 31.2 & 17.2 & 23.6 & 21.6 & 14.3 & 12.1 & 25.7 \\
 & \model{} (3 layers) & \textbf{30.3} & 35.8 & 56.8 & 23.7 & 17.5 & \textbf{52.9} & \textbf{60.6} & \textbf{34.6} & 27.2 & 32.7 & 16.6 & \textbf{23.7} & \textbf{21.7} & \textbf{14.5} & \textbf{13.8} & \textbf{27.5} \\
  & \model{} (2 layers) & 26.7 & 31.8 & 50.8 & 21.3 & 15.7 & 46.9 & 57.1 & 30.8 & 23.2 & 25.4 & 15.0 & 19.0 & 20.8 & 13.1 & 13.5 & 21.2 \\ \bottomrule
\end{tabular}
\end{adjustbox}
\end{table*}
Table \ref{tab:hits10_all} presents HITS@10 results for all the $\mathcal{V}^{inf}/\mathcal{V}_{train}$ ratios. Starting from the middle of the table, for the ratio $150\%$, \model{} (3-layer) performs better than other models in all query types except \emph{1p}. \model{} (2-layer) performs on par with the GNN-QE model. As we examine higher ratios from here, such as $175\%$, $217\%$, and so on. We observe that our models, on average, outperform all the baselines. At these higher ratios, the performance of both variations (3-layer and 2-layer) converge, with the 2-layer version outperforming the 3-layer in the $300\%$ ratio.

Conversely, as we move to lower ratios, $133\%$, $121\%$, and so on, the 2-layered version performs worse than GNN-QE. The average results of the baselines are converging to \model{} (3-layer), with GNN-QE surpassing it for the positive queries (avg$_p$) in $121\%$. NodePiece-QE with GNN outperforms all for the positive queries in $106\%$. 

Table \ref{tab:hits3_all} consists HITS@3 score for all the $\mathcal{V}^{inf}/\mathcal{V}_{train}$ ratios. The results indicate similar trends as HITS@10 scores in Table \ref{tab:hits10_all}.

All these results, in Table \ref{tab:hits10_all} and \ref{tab:hits3_all}, indicate that \model{} performs better as the ratio $\mathcal{V}^{inf}/\mathcal{V}_{train}$ increases ($|\alpha_1| \ll |\alpha_2|$, Section \ref{sec:prelimianries}). This improvement may be attributed to the additional structural information \model{} receives through graph wavelets, particularly when there is limited information due to a small training graph. In extreme cases, when the training graph is significantly smaller than the complete graph ($300\%$ ratio), there is very little information available for training. As a result, our model converges faster, and adding a third layer leads to overfitting on the small training graph, making it inferior to the 2-layered model.

\begin{table*}[t!]
    \centering
    \caption{HITS@100 (\%) score for \emph{inductive} query answering for WikiKG dataset. avg$_p$ is the average of the positive queries.}
    \begin{tabular}{lrrrrrrrrrrr}
        \toprule
        \bf{Model} & \bf{avg$_p$} & \bf{1p} & \bf{2p} & \bf{3p} & \bf{2i} & \bf{3i} & \bf{pi} & \bf{ip} & \bf{2u} & \bf{up}\\
        \midrule
        Edge-type Heuristic & 4.3 & 15.1 & 1.5 & 1.7 & 4.9 & 9.7 & 2.4 & 1.5 & 1.3 & 0.7 \\
        NodePiece-QE & 5.2 & 22.3 & 1.7 & 1.6 & 4.7 & 8.0 & 2.4 & 1.6 & 3.4 & 1.2 \\
        NodePiece-QE w/ GNN & 11.4 & 63.5 & 1.8 & 1.2 & 8.7 & 13.9 & 3.9 & 1.5 & 7.4 & 0.9 \\
        \model{} (2 layers) & \textbf{26.8} & \textbf{69.3} & \textbf{8.0} & \textbf{9.4} & \textbf{94.4} & \textbf{97.2} & \textbf{47.2} & \textbf{27.4} & \textbf{17.5} & \textbf{2.2} \\
        \bottomrule
    \end{tabular}
    \label{tab:wikikg}
\end{table*}

\begin{table*}[t!]
    \centering
    \caption{HITS@10 (\%) score for ablation study on \emph{inductive} query answering for $\mathcal{V}^{inf}/\mathcal{V}_{train} = 175\%$. avg$_p$ is the average of positive queries, and avg is the average of all the queries. Here (w/o lc) is without linear combination, and (w cplx) is with ComplEx link prediction algorithm.}
    \begin{adjustbox}{width=\textwidth}
    \begin{tabular}{lrrrrrrrrrrrrrrrr}
        \toprule
        \bf{Model} & \bf{avg} & \bf{avg$_p$} & \bf{1p} & \bf{2p} & \bf{3p} & \bf{2i} & \bf{3i} & \bf{pi} & \bf{ip} & \bf{2u} & \bf{up} & \bf{2in} & \bf{3in} & \bf{inp} & \bf{pin} & \bf{pni} \\
        \midrule
            \model{} (3 layers) & \textbf{43.5} & \textbf{50.5} & \textbf{65.7} & \textbf{39.5} & \textbf{32.2} & 72.4 & 82.7 & \textbf{57.0} & \textbf{43.2} & 34.5 & \textbf{27.8} & \textbf{29.2} & 41.6 & \textbf{30.9} & \textbf{26.4} & \textbf{26.6} \\
            \model{} (3 layers) (w/o lc) & 42.7 & 49.8 & 64.1 & 38.4 & 30.8 & \textbf{73.0} & \textbf{83.0} & 56.8 & 41.5 & 34.4 & 25.9 & 28.3 & \textbf{41.8} & 30.2 & 24.8 & 24.7\\
             \model{} (3 layers) (w cplx) & 42.3 & 49.5 & 64.1 & 38.7 & 31.5 & 71.9 & 82.3 & 55.8 & 39.9 & \textbf{35.0} & 26.2 & 27.4 & 40.3 & 30.2 & 24.2 & 25.0\\
        \bottomrule
    \end{tabular}
    \end{adjustbox}
    \label{tab:abl}
\end{table*}
At low ratios, most of the necessary information from the complete graph is already present in the training graph. In such cases, additional message-passing layers tend to be more beneficial than merging wavelet embeddings. This is evident, as the 4-layered GNN-QE either shows similar results or outperforms \model{} for lower ratios. However, when the ratio is high, \model{} performs better because any additional information becomes crucial; in this scenario, we obtain that information through graph wavelets. 

Therefore, \model{} is a suitable choice while working with massive graphs, especially when the training graph is much smaller compared to the complete graph. Although \model{} (2 layers) performs inferior in most cases, it remains a good option when the ratio is high or when there are resource constraints since it requires half as many layers as GNN-QE.

Table \ref{tab:wikikg} presents the results of models on the large Wiki-KG dataset. We chose the two-layered \model{} due to its lower memory requirements. We were able to train the model on a GPU, which was previously not possible for GNN-QE due to its high memory footprint~\cite{galkin2022inductive}. \model{} outperforms all baseline models across all query types.  

\subsection{Complexity Analysis}
{\setlength{\parindent}{0pt}\textbf{GE-SpMM.} GE-SpMM~\cite{huang2020ge} proposed an efficient approach for performing message passing through generalized sparse matrix multiplication. It performs in-place operations on GPU threads, resulting in lower memory requirements than traditional methods. NBF-Net used this approach, reducing its space complexity from $O(b|\mathcal{E}|d)$ to $O(b|\mathcal{V}|d)$. Here, $b$ is the batch size, $|\mathcal{E}|$ and $|\mathcal{V}|$ are the number of triples and nodes respectively, and $d$ is the embedding dimension.}

We extend the GE-SpMM method to include in-place operations for graph wavelet embeddings, ensuring the method's compatibility with \msgpass{}. We used this method for Equations \ref{eqn:msgp_re} and \ref{eqn:msgp_im}, improving the combined space complexity from $O(2b|\mathcal{E}|d)$ to $O(b|\mathcal{V}|d+|\mathcal{E}|d)$. In this enhanced version, $|\mathcal{E}|d$ space is shared across the batch, since the wavelet embeddings remain the same. 

NBF-Net contains four message passing layers. Though, \msgpass{} requires an additional $O(|\mathcal{E}|d)$ space per batch at each layer. It saves space of 2 layers in \model{} (2 layers) and 1 layer in \model{} (3 layers).

{\setlength{\parindent}{0pt}\textbf{Laplacian and Wavelet Embeddings.} \model{} needs to compute an additional \emph{KG Laplacian} and \emph{Graph Wavelet Embedding} compared to GNN-QE. We can intelligently distribute the computation for both across $\mathcal{R}$ relations. This decreases the peak memory requirements at any given time. Furthermore, the \emph{Graph Wavelet Embedding} can also be distributed over $\mathcal{V}$, which further minimizes the peak memory requirements. Additionally, we only need to compute $\mathbf{\chi}^r_u$ for the triples present in the KG (Equation \ref{eqn:mp_message}). Lastly, we can always tune the Chebyshev approximation hyperparameter according to the available resources. Thus, these computations can be efficiently managed according to the resources at hand.}

{\setlength{\parindent}{0pt}\textbf{Message Aggregation.}
Both \msgpass{} and NBF-Net use PNA for message aggregation. PNA (Principal Neighborhood Aggregator) aggregates along four types: \emph{mean}, \emph{max}, \emph{min}, and \emph{std}; and it scales them along three metrics: \emph{node degree}, \emph{1}, and \emph{1/node degree}. For NBF-Net, this results in an intermediate embedding dimension of $4 \times 3 \times d$ for each node. Subsequently, this embedding is reduced to $d$ dimensions using a feed-forward layer through matrix multiplication. This includes matrices of dimensions $N\times 12d$ and $12d\times d$.}

\msgpass{} has separate real and imaginary messages. This leads to matrix multiplications with dimensions $N\times 12(d/2)$ and $12(d/2)\times (d/2)$ for each of the real and imaginary parts. By dividing the message into two smaller matrix multiplications instead of one larger one, \msgpass{} saves both space and time.

\subsection{Ablation Study}
In this study, we compare \model{} against two configurations as part of an ablation analysis. 

In the first test, instead of using a linear combination and introducing additional parameters, we directly multiply the wavelet embedding with the node and relation embeddings. So, Equation \ref{eqn:mp_message} is modified to the following relation.
\begin{align}\label{eqn:mp_message_abl}
    MES&SAGE (h^{(t-1)}_{x_{re}}, w_{q_{re}}(x, r, v),\chi^r_{x_{re}}) \nonumber\\
    &= h^{ (t-1)}_{x_{re}} \odot (W_rq + b_r)_{re} \odot \chi^r_{x_{re}}
\end{align}

In the second configuration, we replaced the DistMult link prediction method in (Equation \ref{eqn:mp_message}) with the ComplEx~\cite{trouillon2016complex} link prediction method. Since graph wavelet embeddings are complex, this was intended to test our model in a complex-plane environment. 

Table \ref{tab:abl} presents results for \model{} (3 layers) \textit{without linear combination} (w/o lc) and \textit{with ComplEx} (w cplx). The results consists of HITS@10 scores for $\mathcal{V}^{inf}/\mathcal{V}_{train} = 175\%$, Fb15k-(237) dataset. We can see from the table that the original \model{} performs better, on average and across most query structures, from both the configurations.

\subsection{Space and Run-Time Usage}
All the following details are for $175\%$ ratio of FB15k-(237), where all the nodes in $\mathcal{V}_{train}$ are used in a single training graph (the first dataset). 

For the GNN-QE, the runtime for one training epoch was 7 minutes and 39 seconds, and GPU memory usage was 20.16 GB. In comparison, the \model{} (3 layers) model took 7 minutes and 9 seconds and used 7.57 GB of memory. Further, for the \model{} (2 layers) model, the runtime was 5 minutes 31 seconds and required 5.63GB of GPU memory.

Additionally, the one-time preprocessing time to generate the graph wavelet embedding (including the KG Laplacian) was 10 minutes and 50 seconds. This preprocessing time was common for both \model{} (2 layers) and \model{} (3 layers), as the hyperparameters for graph wavelet embeddings were identical for both versions.

From this information, we can conclude that our model is more memory and time efficient than GNN-QE, aside from the one-time preprocessing step for graph wavelet embeddings.

\section{Conclusion}
In this work, we present a scalable method for \emph{inductive} logical query answering over KGs. The model is particularly well-suited for scenarios where the training graph is much smaller than the complete graph. It can effectively handle large graphs containing millions of nodes (Wiki-KG). Although the average query score for \model{} is the highest for WIKI-KG dataset, there is still scope for improvement, which could be considered for future work. A future direction could be to explore distributed training for large KGs. Another direction is to develop models that can handle other operators used in query answering.

\section*{Acknowledgments}
ChatGPT and Gemini were used for basic text and code editing.

\bibliographystyle{IEEEtran}
\bibliography{sample-base}
\begin{IEEEbiography}[{\includegraphics[width=1in,height=1.25in,clip,keepaspectratio]{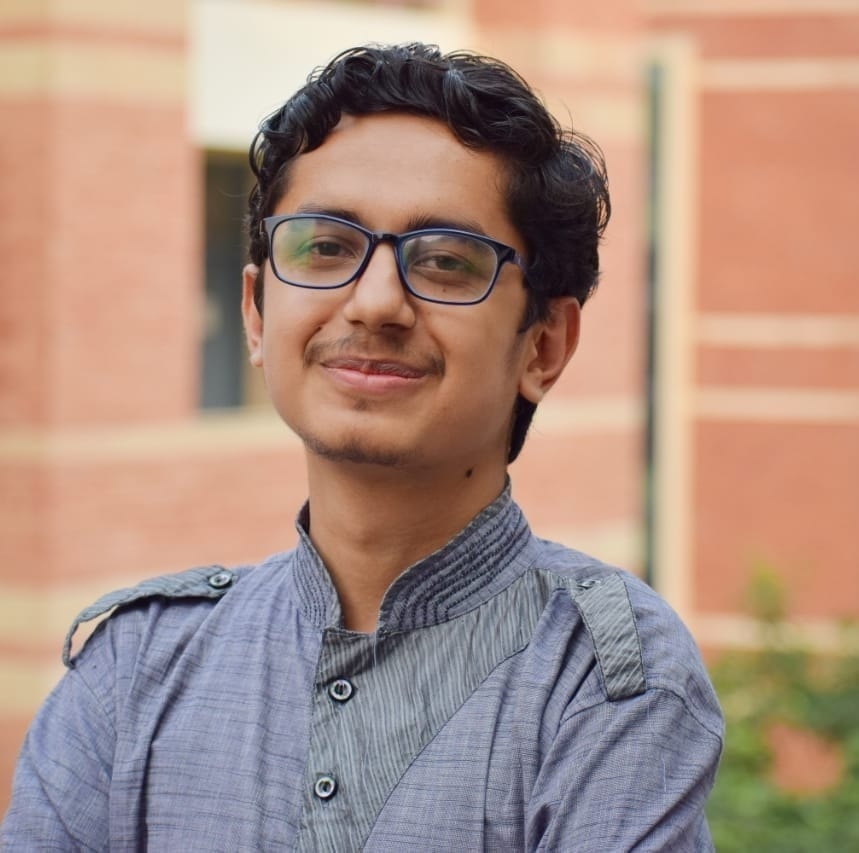}}]{Mayank Kharbanda} received his bachelor’s and master’s degrees in computer science from the University of Delhi, India, in 2018 and 2020, respectively. He is currently working towards a PhD in the Department of Computer Science at IIIT Delhi. He is also a guest at Vrije Universiteit Amsterdam, the Netherlands. His research interests include Knowledge Graphs, Graph Representation Learning, and discrete mathematics. He is a UGC-JRF/SRF fellow. More information at \url{https://mayankkharbanda.github.io/}.
\end{IEEEbiography}

\begin{IEEEbiography}[{\includegraphics[width=1in,height=1.25in,clip,keepaspectratio]{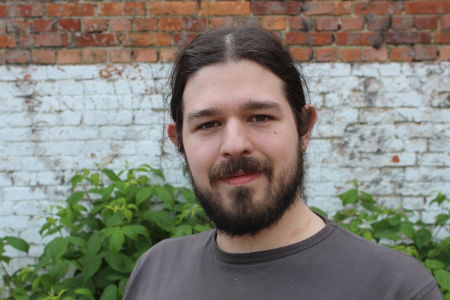}}]{Michael Cochez} 
is Principal Investigator (PS Fellow) at the ELLIS Institute Finland and Professor at Åbo Akademi University, Turku, Finland. He is also part-time at Vrije Universiteit Amsterdam, the Netherlands.
He was a postdoc at Fraunhofer FIT, Germany and obtained his PhD from the University of jyväskylä, Finland.
More information at \url{http://www.cochez.nl}.
\end{IEEEbiography}

\begin{IEEEbiography}[{\includegraphics[width=1in,height=1.25in,clip,keepaspectratio]{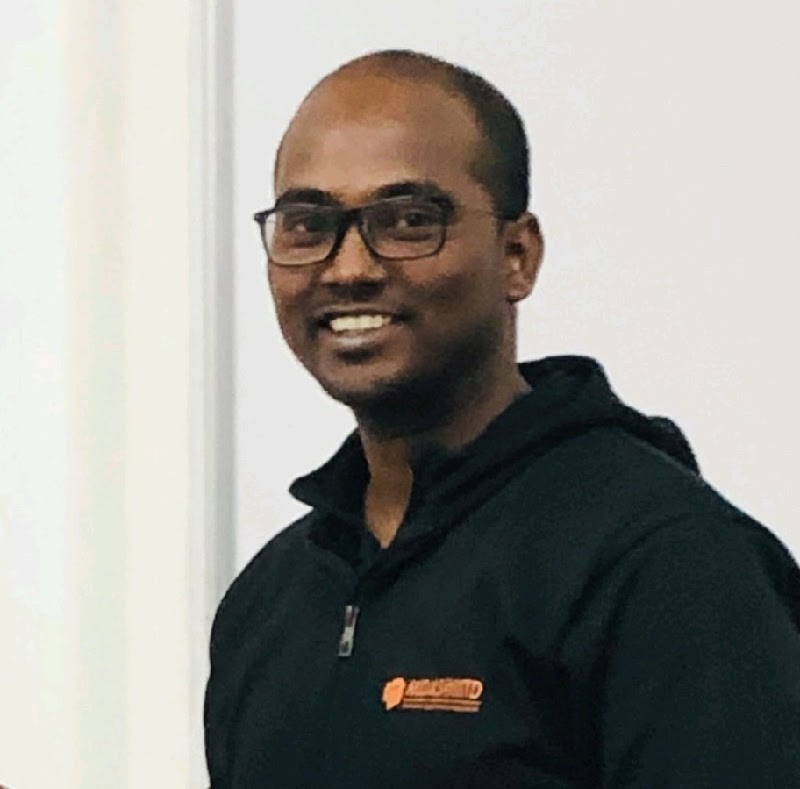}}]{Rajiv Ratn Shah} is an Associate Professor in the Department of Computer Science and Engineering at IIIT-Delhi and Head of the MIDAS Research Lab. Prior to joining IIIT-Delhi, he was a Research Fellow at Singapore Management University and earned his Ph.D. from the National University of Singapore. His research interests include multimodal AI, natural language processing, computer vision, and human-centered computing. More information at \url{https://midas.iiitd.ac.in/}.
\end{IEEEbiography}
\begin{IEEEbiography}[{\includegraphics[width=1in,height=1.25in,clip,keepaspectratio]{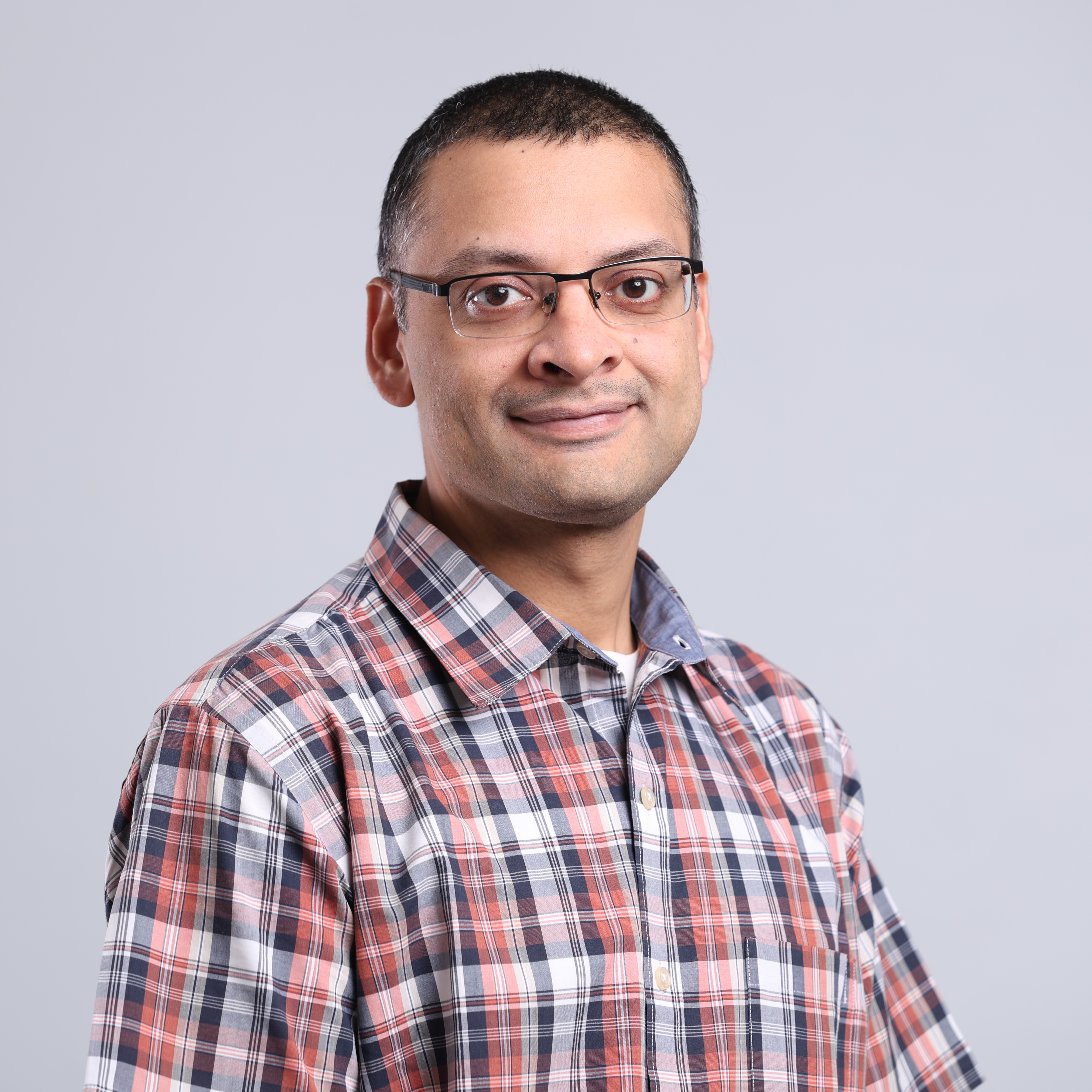}}]{Raghava Mutharaju} is an Associate Professor in the Data Science department and is part of the Mehta Family School of Data Science and AI at IIT Palakkad, Kerala, India. He leads the Knowledgeable Computing and Reasoning (KRaCR) Lab. Earlier, we worked at  IIIT-Delhi, GE Research, IBM Research, Bell Labs, and Xerox Research. His research interests include Knowledge Graphs, ontology modelling, and explainable AI. He has published at several venues, such as AAAI, IJCAI, ACL, and ISWC. More information at \url{https://kracr.github.io/}.
\end{IEEEbiography}

\vfill

\newpage
\appendix

\subsection{KG Laplacian}
The theorems here are directly extended from MagNet~\cite{zhang2021magnet}, and mentioned here for completeness. We first prove that the KG Laplacian for relation $r$ is positive, semi-definite for both the normalized and the un-normalized Laplacian. We then show that the normalized KG Laplacian lies in $[0, 2]$, as for the Laplacian of undirected graphs. Other properties of the KG Laplacian can be extended from these proofs with the context of~\cite{zhang2021magnet}.
\begin{theorem}\label{thm: posdef}
Let $G(\mathcal{V},\mathcal{E},\mathcal{R})$ is a KG with a set of nodes $\mathcal{V}$, a set of relations $\mathcal{R}$, and a triplet set 
$\mathcal{E} = \{(e_s, r, e_o) | e_s, e_o \in \mathcal{V}, r \in \mathcal{R}\}$, and $|\mathcal{V}|=N$. Then, for all $g\geq 0,$ and $r \in \mathcal{R}$ $\boldsymbol{L}_{un}^{r(g)}$ and their normalized counterpart 
$\boldsymbol{L}_{n}^{r(g)}$ are positive semidefinite.
\end{theorem}

\begin{proof}
Let $\mathbf{x}\in\mathbb{C}^N$. For a relation $r$ we first note that since $\mathbf{L}^{r(g)}_{un}$ is Hermitian we have $\text{Imag}(\mathbf{x}^\dagger\mathbf{L}^{(g)}_{un}\mathbf{x})=0$. Next, we use the definition of $\mathbf{D}^r_s$ and the fact that $\mathbf{A}^r_s$ is symmetric to observe that \cite{zhang2021magnet,he2022msgnn}.
\begingroup
\scriptsize            
\allowdisplaybreaks
\begin{align}
2\text{Real}&\left(\mathbf{x}^\dagger\mathbf{L}^{r(g)}_{un}\mathbf{x}\right)
    =2\sum_{u,v=1}^N\mathbf{D}^r_s(u,v)\mathbf{x}(u)\overline{\mathbf{x}(v)}\nonumber\\
    &-2\sum_{u,v=1}^N \mathbf{A}^r_s(u,v)\mathbf{x}(u)\overline{\mathbf{x}(v)} \cos(\boldsymbol{\Theta}_r^{(g)} (u,v))\nonumber\\
    =2&\sum_{u=1}^N \mathbf{D}^r_s(u,u)\mathbf{x}(u)\overline{\mathbf{x}(u)}\nonumber\\
    &-2\sum_{u,v=1}^N \mathbf{A}^r_s(u,v)\mathbf{x}(u)\overline{\mathbf{x}(v)} \cos(\boldsymbol{\Theta}_r^{(g)} (u,v))\nonumber\\
    =2&\sum_{u,v=1}^N \mathbf{A}^r_s(u,v)|\mathbf{x}(u)|^2\nonumber\\
    &-2\sum_{u,v=1}^N \mathbf{A}^r_s(u,v)\mathbf{x}(u)\overline{\mathbf{x}(v)} \cos(\boldsymbol{\Theta}_r^{(g)} (u,v))\nonumber\\
    =&\sum_{u,v=1}^N \mathbf{A}^r_s(u,v)|\mathbf{x}(u)|^2+
    \sum_{u,v=1}^N \mathbf{A}^r_s(v,u)|\mathbf{x}(v)|^2\nonumber\\
    &-2\sum_{u,v=1}^N \mathbf{A}^r_s(u,v)\mathbf{x}(u)\overline{\mathbf{x}(v)} \cos(\boldsymbol{\Theta}_r^{(g)} (u,v))\nonumber\\
    =&\sum_{u,v=1}^N \mathbf{A}^r_s(u,v)|\mathbf{x}(u)|^2+
    \sum_{u,v=1}^N \mathbf{A}^r_s(u,v)|\mathbf{x}(v)|^2\nonumber\\
    &-2\sum_{u,v=1}^N \mathbf{A}^r_s(u,v)\mathbf{x}(u)\overline{\mathbf{x}(v)} \cos(\boldsymbol{\Theta}_r^{(g)} (u,v))\nonumber\\
    =&\sum_{u,v=1}^N \mathbf{A}^r_s(u,v)\nonumber\\
    &\left(|\mathbf{x}(u)|^2+
    |\mathbf{x}(v)|^2-2\mathbf{x}(u)\overline{\mathbf{x}(v)} \cos(\boldsymbol{\Theta}_r^{(g)} (u,v))\right)\label{eqn: quad form}\\
    \geq&\sum_{u,v=1}^N \mathbf{A}^r_s(u,v)\left(|\mathbf{x}(u)|^2+
    |\mathbf{x}(v)|^2-2|\mathbf{x}(u)||\mathbf{x}(v)|\right)\nonumber\\
    =&\sum_{u,v=1}^N \mathbf{A}^r_s(u,v)(|\mathbf{x}(u)|-|\mathbf{x}(v)|)^2\geq0.\nonumber
\end{align}
\endgroup
Thus, each $\mathbf{L}_{un}^{r(g)}$ is positive semidefinite, for $r \in R$. For the normalized magnetic Laplacian, we note that
\begin{equation}\label{eqn: divide by D}
\mathbf{L}^{r(g)}_{n}=\boldsymbol{D}_z^{-1/2}\mathbf{L}^{r(g)}_{un}\boldsymbol{D}_z^{-1/2},\;\;\;\mathbf{D}_z(u,u) = \sum_{r\in\mathcal{R}}\sum_{v \in \mathcal{V}} \mathbf{A}^r_s(u, v)
\end{equation}
We are using $D_z$ for normalization in place of $D_s^r$, to get normalized value influenced by the complete graph rather than a particular relation.
Thus, letting $\mathbf{y}=\mathbf{D}_z^{-1/2}\mathbf{x},$ the fact that $\mathbf{D}_z$ is diagonal implies
\begin{equation*}
\mathbf{x}^\dagger\mathbf{L}^{r(g)}_n\mathbf{x}=\mathbf{x}^\dagger\boldsymbol{D}_z^{-1/2}\mathbf{L}^{r(g)}_{un}\boldsymbol{D}_z^{-1/2}\mathbf{x}=\mathbf{y}^\dagger\mathbf{L}^{r(g)}_{un}\mathbf{y}\geq 0.
\end{equation*}
\end{proof}
\begin{theorem}\label{thm: normal02} Let $G(\mathcal{V},\mathcal{E},\mathcal{R})$ is a KG with a set of nodes $\mathcal{V}$, a set of relations $\mathcal{R}$, and a triplet set $\mathcal{E} = \{(e_s, r, e_o) | e_s, e_o \in \mathcal{V}, r \in \mathcal{R}\}$. Then, for all $g\geq 0$, and $r \in \mathcal{R}$ the eigenvalues of the normalized magnetic Laplacian $\mathbf{L}_{n}^{r(g)}$ are contained in the interval $[0,2]$.
\end{theorem}

\begin{proof}
By Theorem \ref{thm: posdef}, we know that $\mathbf{L}^{r(g)}_{n}$ has real, nonnegative eigenvalues. Therefore, we need to show that the lead eigenvalue, $\lambda_{N},$ is less than or equal to 2.
The Courant-Fischer theorem shows that 
\begin{equation*}
    \lambda_N=\max_{\mathbf{x}\neq 0}\frac{\mathbf{x}^\dagger\mathbf{L}^{r(g)}_{n}\mathbf{x}}{\mathbf{x}^\dagger\mathbf{x}}.
\end{equation*}
Therefore, using \eqref{eqn: divide by D} and setting $\mathbf{y}=\mathbf{D}_z^{-1/2}\mathbf{x}$, we have 
\begin{equation*}
    \lambda_N=\max_{\mathbf{x}\neq 0}\frac{\mathbf{x}^\dagger\mathbf{D}_z^{-1/2}\mathbf{L}^{r(g)}_{un}\mathbf{D}_z^{-1/2}\mathbf{x}}{\mathbf{x}^\dagger\mathbf{x}} =\max_{\mathbf{y}\neq 0}\frac{\mathbf{y}^\dagger\mathbf{L}^{(g)}_{un}\mathbf{y}}{\mathbf{y}^\dagger\mathbf{D}_z\mathbf{y}}.
\end{equation*}
First, we observe that since $\mathbf{D}_z$ is diagonal, we have 
\begin{equation*}
\mathbf{y}^\dagger\mathbf{D}_z\mathbf{y}=\sum_{u,v=1}^N \mathbf{D}_z(u,v)\mathbf{y}(u)\overline{\mathbf{y}(v)}=\sum_{u=1}^N \mathbf{D}_z(u,u)|\mathbf{y}(u)|^2
\end{equation*}
Next, we note that by \eqref{eqn: quad form}, we have
{\scriptsize
\begin{align*}
    \mathbf{y}^\dagger&\mathbf{L}^{r(g)}_{un}\mathbf{y}
    =\sum_{u,v=1}^N \mathbf{A}^r_s(u,v)(|\mathbf{x}(u)|^2+
    |\mathbf{x}(v)|^2\\&-2\mathbf{x}(u)\overline{\mathbf{x}(v)} \cos(\boldsymbol{\Theta}_r^{(g)} (u,v)))\\
    &\leq \frac{1}{2}\sum_{u,v=1}^N \mathbf{A}^r_s(u,v)(|\mathbf{x}(u)|+|\mathbf{x}(v)|)^2\\&\leq  \sum_{u,v=1}^N \mathbf{A}^r_s(u,v)(|\mathbf{x}(u)|^2+|\mathbf{x}(v)|^2).
\end{align*}}
Therefore, since $\mathbf{A^r_s}$ is symmetric, we have 
{\scriptsize
\begin{align*}
    \mathbf{y}^\dagger\mathbf{L}^{r(g)}_{un}&\mathbf{y} 
    \leq 2\sum_{u,v=1}^N \mathbf{A}^r_s(u,v)|\mathbf{x}(u)|^2=2\sum_{u=1}^N|\mathbf{x}(u)|^2\left(\sum_{v=1}^N \mathbf{A}^r_s(u,v)\right)\\
    &=2\sum_{u=1}^N\mathbf{D}^r_s(u,u)|\mathbf{x}(u)|^2\leq 2\sum_{u=1}^N\mathbf{D}_z(u,u)|\mathbf{x}(u)|^2=2    \mathbf{y}^\dagger\mathbf{D}_z\mathbf{y}.
\end{align*}}
\end{proof}
\setcounter{table}{0}
\subsection{Query Statistics and Results}
\begin{table*}[!t]
\centering
\caption{Number of queries generated for each dataset ratio and query type.}
\label{tab:all_queries}
\begin{adjustbox}{width=\textwidth}
\begin{tabular}{lrrrrrrrrrrrrrrrr}\toprule
Ratio & Graph & \multicolumn{1}{c}{\textbf{1p}} & \multicolumn{1}{c}{\textbf{2p}} & \multicolumn{1}{c}{\textbf{3p}} & \multicolumn{1}{c}{\textbf{2i}} & \multicolumn{1}{c}{\textbf{3i}} & \multicolumn{1}{c}{\textbf{pi}} & \multicolumn{1}{c}{\textbf{ip}} & \multicolumn{1}{c}{\textbf{2u}} & \multicolumn{1}{c}{\textbf{up}} & \multicolumn{1}{c}{\textbf{2in}} & \multicolumn{1}{c}{\textbf{3in}} & \multicolumn{1}{c}{\textbf{inp}} & \multicolumn{1}{c}{\textbf{pin}} & \multicolumn{1}{c}{\textbf{pni}} \\\midrule
\multirow{3}{*}{300\%} &training & 10,066 &15,000 &15,000 &40,000 &40,000 &50,000 &50,000 &50,000 &50,000 &10,000 &10,000 &30,000 &30,000 &30,000 \\
&validation & 17,088 &50,000 &50,000 &50,000 &50,000 &50,000 &50,000 &50,000 &50,000 &10,000 &10,000 &10,000 &10,000 &10,000 \\
&test & 15,851 &50,000 &50,000 &50,000 &50,000 &50,000 &50,000 &50,000 &50,000 &10,000 &10,000 &10,000 &10,000 &10,000 \\ 
\midrule
\multirow{3}{*}{217\%} &training & 18,453 &15,000 &15,000 &40,000 &40,000 &50,000 &50,000 &50,000 &50,000 &15,000 &15,000 &50,000 &50,000 &40,000 \\
&validation & 17,516 &50,000 &50,000 &50,000 &50,000 &50,000 &50,000 &50,000 &50,000 &10,000 &10,000 &10,000 &10,000 &10,000 \\
&test & 17,153 &50,000 &50,000 &50,000 &50,000 &50,000 &50,000 &50,000 &50,000 &10,000 &10,000 &10,000 &10,000 &10,000 \\
\midrule
\multirow{3}{*}{175\%} &training & 30,903 &15,000 &15,000 &40,000 &40,000 &50,000 &50,000 &50,000 &50,000 &15,000 &15,000 &50,000 &50,000 &40,000 \\
&validation & 17,797 &50,000 &50,000 &50,000 &50,000 &50,000 &50,000 &50,000 &50,000 &10,000 &10,000 &10,000 &10,000 &10,000 \\
&test & 17,188 &50,000 &50,000 &50,000 &50,000 &50,000 &50,000 &50,000 &50,000 &10,000 &10,000 &10,000 &10,000 &10,000 \\
\midrule
\multirow{3}{*}{150\%} & training & 42,203 &30,000 &30,000 &50,000 &50,000 &50,000 &50,000 &50,000 &50,000 &30,000 &30,000 &50,000 &50,000 &50,000 \\
&validation & 16,916 &50,000 &50,000 &50,000 &50,000 &50,000 &50,000 &50,000 &50,000 &10,000 &10,000 &10,000 &10,000 &10,000 \\
&test & 17,379 &50,000 &50,000 &50,000 &50,000 &50,000 &50,000 &50,000 &50,000 &10,000 &10,000 &10,000 &10,000 &10,000 \\ 
\midrule
\multirow{3}{*}{134\%} &training & 57,242 &50,000 &50,000 &50,000 &50,000 &50,000 &50,000 &50,000 &50,000 &50,000 &40,000 &50,000 &50,000 &50,000 \\
&validation & 15,614 &50,000 &50,000 &50,000 &50,000 &50,000 &50,000 &50,000 &50,000 &10,000 &10,000 &10,000 &10,000 &10,000 \\
&test & 15,544 &50,000 &50,000 &50,000 &50,000 &50,000 &50,000 &50,000 &50,000 &10,000 &10,000 &10,000 &10,000 &10,000 \\
\midrule
\multirow{3}{*}{122\%} &training & 73,373 &50,000 &50,000 &50,000 &50,000 &50,000 &50,000 &50,000 &50,000 &50,000 &40,000 &50,000 &50,000 &50,000 \\
&validation & 13,842 &50,000 &50,000 &50,000 &50,000 &50,000 &50,000 &50,000 &50,000 &5,000 &5,000 &5,000 &5,000 &5,000 \\
&test & 13,556 &50,000 &50,000 &50,000 &50,000 &50,000 &50,000 &50,000 &50,000 &5,000 &5,000 &5,000 &5,000 &5,000 \\ 
\midrule
\multirow{3}{*}{113\%} &training & 93,861 &50,000 &50,000 &50,000 &50,000 &50,000 &50,000 &50,000 &50,000 &50,000 &50,000 &50,000 &50,000 &50,000 \\
&validation & 10,193 &10,000 &10,000 &10,000 &10,000 &10,000 &10,000 &10,000 &10,000 &5,000 &5,000 &5,000 &5,000 &5,000 \\
&test & 10,412 &10,000 &10,000 &10,000 &10,000 &10,000 &10,000 &10,000 &10,000 &5,000 &5,000 &5,000 &5,000 &5,000 \\
\midrule
\multirow{3}{*}{106\%} &training & 109,976 &50,000 &50,000 &50,000 &50,000 &50,000 &50,000 &50,000 &50,000 &50,000 &50,000 &50,000 &50,000 &50,000 \\
&validation & 5,598 &10,000 &10,000 &10,000 &10,000 &10,000 &10,000 &10,000 &10,000 &1,000 &1,000 &1,000 &1,000 &1,000 \\
&test & 6,948 &10,000 &10,000 &10,000 &10,000 &10,000 &10,000 &10,000 &10,000 &1,000 &1,000 &1,000 &1,000 &1,000 \\ 
\midrule
        \multicolumn{16}{c}{\emph{WikiKG}} \\ \midrule
\multirow{3}{*}{133\%} & training & 10,000 & 10,000 & 10,000 & 10,000 & 10,000 & 10,000 & 10,000 & 10,000 & 10,000 & - & - & - & - & -  \\        &validation & 10,000 & 10,000 & 10,000 & 10,000 & 10,000 & 10,000 & 10,000 & 10,000 & 10,000 & - & - & - & - & - \\
&test & 10,000 & 10,000 & 10,000 & 10,000 & 10,000 & 10,000 & 10,000 & 10,000 & 10,000 & - & - & - & - & - \\ 
\bottomrule
\end{tabular}
\end{adjustbox}
\end{table*}
The number of queries generated for each dataset are provided in Table \ref{tab:all_queries}. The statistics mentioned as the train graph are for the graph $\mathcal{V}_{train_0}$. 
\end{document}